\documentclass[11pt,english]{article}
\usepackage{ae,aecompl}
\usepackage[T1]{fontenc}
\usepackage[latin9]{inputenc}
\usepackage{geometry}
\usepackage{float}
\usepackage{amsmath}
\usepackage{amsthm}
\usepackage{amssymb}
\usepackage{graphicx}
\usepackage{esint}

\makeatletter

\providecommand{\tabularnewline}{\\}
\floatstyle{ruled}
\newfloat{algorithm}{tbp}{loa}
\providecommand{\algorithmname}{Algorithm}
\floatname{algorithm}{\protect\algorithmname}

\theoremstyle{plain}
\newtheorem{thm}{\protect\theoremname}
  \theoremstyle{plain}
  \newtheorem{lem}[thm]{\protect\lemmaname}

\makeatother

\usepackage{babel}
  \providecommand{\lemmaname}{Lemma}
\providecommand{\theoremname}{Theorem}

\begin{document}

\title{A spatial compositional model (SCM) for linear unmixing and endmember
uncertainty estimation}

\author{Yuan Zhou, Anand Rangarajan and Paul Gader\\
Dept. of Computer and Information Science and Engineering\\
University of Florida, Gainesville, FL, USA\\
E-mail: \textsf{\{yuan,anand,pgader\}@cise.ufl.edu}}
\maketitle
\begin{abstract}
The normal compositional model (NCM) has been extensively used in
hyperspectral unmixing. However, most of the previous research has
focused on estimation of endmembers and/or their variability. Also,
little work has employed spatial information in NCM. In this paper,
we show that NCM can be used for calculating the \emph{uncertainty}
of the estimated endmembers with spatial priors incorporated for better
unmixing. This results in a spatial compositional model (SCM) which
features (i) spatial priors that force neighboring abundances to be
similar based on their pixel similarity and (ii) a posterior that
is obtained from a likelihood model which does not assume pixel independence.
The resulting algorithm turns out to be easy to implement and efficient
to run. We compared SCM with current state-of-the-art algorithms on
synthetic and real images. The results show that SCM can in the main
provide more accurate endmembers and abundances. Moreover, the estimated
uncertainty can serve as a prediction of endmember error under certain
conditions.
\end{abstract}

\section{Introduction}

Hyperspectral image unmixing has received wide attention in the remote
sensing, signal and image processing communities \cite{keshava2002spectral,bioucas2012hyperspectral}.
The widely researched model in this area is the \emph{linear mixing
model} (LMM). Assume we have a hyperspectral image $I\left(\mathbf{x}\right):\mathcal{D}\to L^{2}\left(\mathbb{R}^{+}\right)$
where $\mathcal{D}\subset\mathbb{R}^{2}$ is the image domain with
$L^{2}\left(\mathbb{R}^{+}\right)$ denoting the space of square integrable
functions on the positive part of the real line. LMM assumes that
the spectral measurement of each pixel $g\left(\mathbf{x},\lambda\right)$,
with $\lambda$ being the wavelength, is a non-negative linear combination
of the spectral signature of some pure materials, called \emph{endmembers},
$f_{j}\left(\lambda\right):\mathbb{R}^{+}\to\mathbb{R}^{+}$. The
governing equation is 
\begin{gather}
g\left(\mathbf{x},\lambda\right)=\sum_{j=1}^{M}f_{j}\left(\lambda\right)\alpha_{j}\left(\mathbf{x}\right)+n\left(\mathbf{x},\lambda\right)\label{eq:LMM1}\\
\forall\mathbf{x}\in\mathcal{D},\,\sum_{j=1}^{M}\alpha_{j}\left(\mathbf{x}\right)=1\nonumber 
\end{gather}
where $M$ is the number of endmembers, $\alpha_{j}\left(\mathbf{x}\right):\mathcal{D}\to\mathbb{R}^{+}$
is the fractional \emph{abundance map} of the \emph{j}th endmember
and satisfies the positivity and sum-to-one (simplex) constraints,
and $n\left(\mathbf{x},\lambda\right):\mathcal{D}\times\mathbb{R}^{+}\to\mathbb{R}$
is a small, additive perturbation (noise). As a result, the pixels
generated by this model form a simplex in an infinite dimensional
vector space whose vertices are the endmembers. If we discretize $f_{j}\left(\lambda\right)$
into $B$ bands and get $\mathbf{m}_{j}\in\mathbb{R}^{B}$ as the
discretized value, and further discretize $\mathcal{D}$ into $N$
locations, equation~(\ref{eq:LMM1}) states that the spectrum of
the \emph{i}th pixel can be represented by 
\[
\mathbf{y}_{i}=\mathbf{M}^{T}\boldsymbol{\alpha}_{i}+\mathbf{n}_{i}
\]
where $\mathbf{M}=\left[\mathbf{m}_{1},...,\mathbf{m}_{M}\right]^{T}$,
$\boldsymbol{\alpha}_{i}=\left[\alpha_{i1},...,\alpha_{iM}\right]^{T}$,
$\mathbf{y}_{i}\in\mathbb{R}^{B}$, $\mathbf{n}_{i}\in\mathbb{R}^{B}$.
Combining the above equation for all the pixels, we have the following
equation for LMM:
\begin{gather}
\mathbf{Y=AM+N}\label{eq:Y=00003DAM+N}
\end{gather}
where $\mathbf{Y}\in\mathbb{R}^{N\times B}$, $\mathbf{A}\in\mathbb{R}^{N\times M}$,
$\mathbf{M}\in\mathbb{R}^{M\times B}$, $\mathbf{N}\in\mathbb{R}^{N\times B}$.

The linear unmixing problem is to retrieve $\mathbf{A}$ and $\mathbf{M}$
given $\mathbf{Y}$. This is an ill-posed inverse problem as it can
have an infinite number of solutions. Figure~\ref{fig:difficulties}
shows the difficulties stemming from this underdetermined nature:
Figure~\ref{fig:difficulties}(a) shows the pixels (gray area) generated
by 2 endmembers $\mathbf{m}_{1}$, $\mathbf{m}_{2}$ with $\boldsymbol{\alpha}=(0.5,0.5)^{T}$
for all pixels. Clearly it is not possible to retrieve $\mathbf{A}$
and $\mathbf{M}$ given $\mathbf{Y}$ in this case. As shown in the
figure, we can have $\mathbf{m}_{1}^{\prime}$, $\mathbf{m}_{2}^{\prime}$
or $\mathbf{m}_{1}^{\prime\prime}$, $\mathbf{m}_{2}^{\prime\prime}$
generate the same pixels. In fact, we can have the whole Euclidean
space for the endmembers. In Figure~\ref{fig:difficulties}(b), the
abundances $\boldsymbol{\alpha}$ ranges from $(0.7,0.3)$ to $(0.2,0.8)$.
Here, we can determine that the endmembers should lie in a line that
fits the pixels. However, we still cannot determine the specific position
of the endmembers without other information. For example, $\mathbf{m}_{1}^{\prime}$,
$\mathbf{m}_{2}^{\prime}$ or $\mathbf{m}_{1}^{\prime\prime}$, $\mathbf{m}_{2}^{\prime\prime}$
can both serve as the endmember set. In Figure~\ref{fig:difficulties}(c),
we have the information that the abundances range from $(1,0)$ to
$(0,1)$. Now we can find the endmembers as we are not only given
the line, but also the relative position of the endmembers to the
boundary of the pixels. This is the only case where we have a unique
solution which corresponds to the original endmembers. The intuition
(stemming from this observation) that the endmembers should tightly
surround the pixels has been extensively used in the literature, in
the form of minimal volume \cite{miao2007endmember}, pure pixels
\cite{nascimento2005vertex}, or pairwise closeness \cite{berman2004ice,zare2013piecewise}.
Figure~\ref{fig:difficulties}(d) shows another interesting problem.
Suppose we have two endmember sets, $\mathbf{m}_{1}$, $\mathbf{m}_{2}$
and $\mathbf{m}_{3}$, $\mathbf{m}_{4}$ and we can find them. How
can we then determine the abundance $\boldsymbol{\alpha}$ of the
pixel in the intersection given these endmembers? Should it be a linear
combination of $\mathbf{m}_{1}$ and $\mathbf{m}_{2}$ or a linear
combination of $\mathbf{m}_{3}$ and $\mathbf{m}_{4}$, or a combination
of all of them? However, if the pixels from an endmember set lies
in one region while those from the other set lie in another region,
the abundances in the intersection can by easily identified by the
spatial location. This implies that spatial information should be
used in the unmixing process.

\begin{figure}
\begin{centering}
\includegraphics[width=1\textwidth]{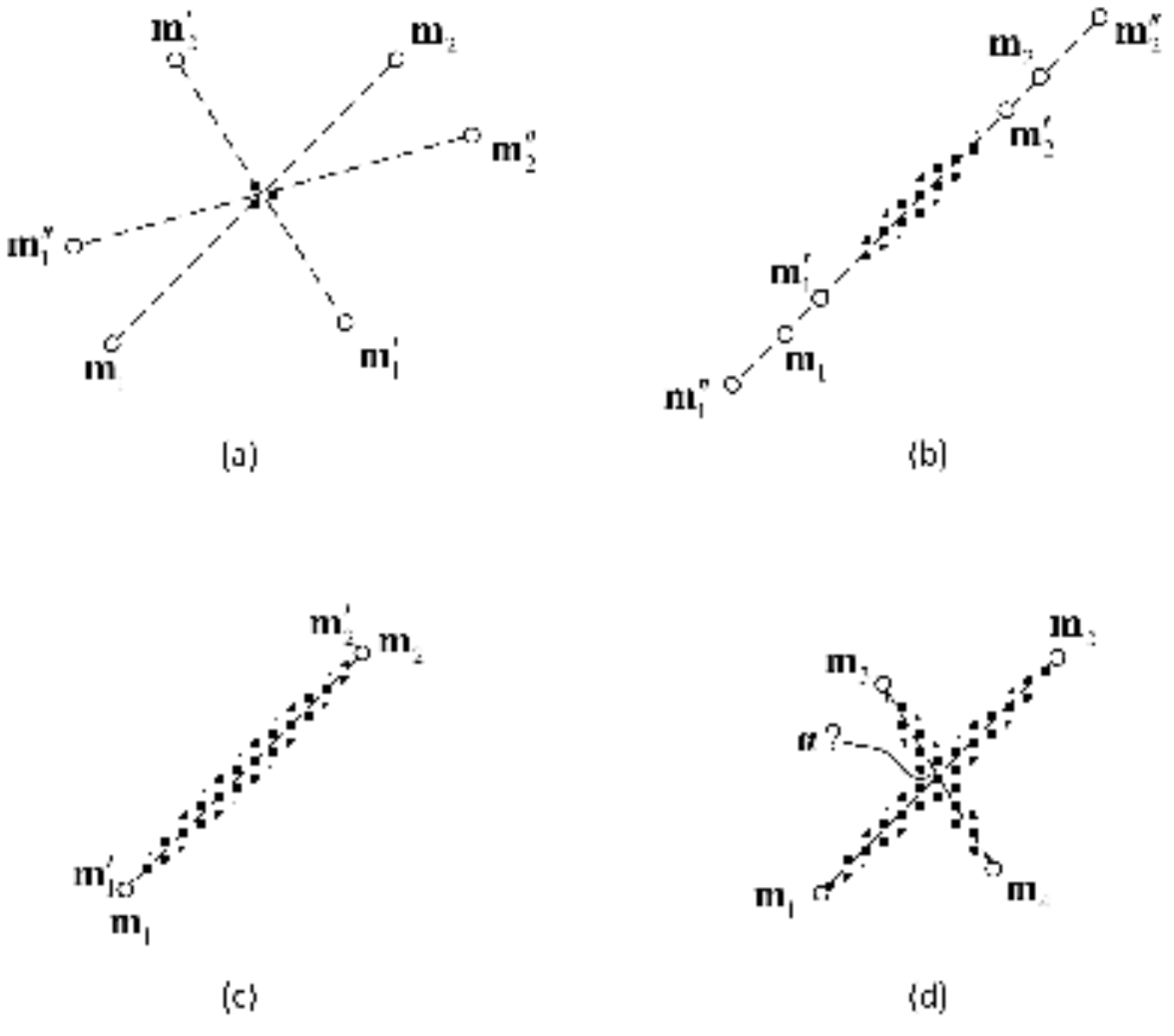}
\par\end{centering}

\caption{Difficulties in the linear unmixing problem: $\mathbf{m}_{1}$, $\mathbf{m}_{2}$
are the endmembers that generate the pixel data. $\mathbf{m}_{1}^{\prime}$,
$\mathbf{m}_{2}^{\prime}$ and $\mathbf{m}_{1}^{\prime\prime}$, $\mathbf{m}_{2}^{\prime\prime}$
are the possible endmembers that can be inferred from the data. In
(a), (b) are 2 cases where the true endmembers can not be estimated
while (c) contains a case where they can be estimated under some assumptions.
In (d), we show the importance of spatial location to abundance estimation.}
\label{fig:difficulties}
\end{figure}

The methods developed to solve this problem may be mainly categorized
into geometrical, statistical and sparse regression based approaches
\cite{bioucas2012hyperspectral}. For example, vertex component analysis
(VCA) assumes the endmembers are present in the image pixels \cite{nascimento2005vertex},
iterative constrained endmembers (ICE) minimizes the least squares
error under the pairwise closeness constraint \cite{berman2004ice},
minimum volume constrained nonnegative matrix factorization (MVC-NMF)
minimizes the same error and the volume of the simplex \cite{miao2007endmember},
piecewise convex multiple-model endmember detection (PCOMMEND) minimizes
the least square errors of separate convex sets \cite{zare2013piecewise}. 

Besides these methods that rely only on independent pixels, some recent
work introduced spatial information to aid the unmixing process \cite{jia2009constrained,eches2011enhancing,iordache2012total}.
For example, in \cite{jia2009constrained} two smoothness terms for
abundances and endmembers were proposed to utilize the spatial information
in terms of wavelength proximity and pixel location. In \cite{eches2011enhancing}
a Markov Random Field (MRF), Potts-Markov model, was used to model
the partitioning of the image that can help the unmixing process.
Sampling methods were used to infer the unknown parameters. In \cite{iordache2012total},
a constraint that minimizes the $L_{1}$ norm of the differences between
neighboring abundances was proposed to impose spatial correlation.

Another type of method is based on modeling the likelihood using Gaussian
density functions, also known as the \emph{normal compositional model}
(NCM) \cite{eches2010bayesian,eches2010estimating,zare2010pce,stein2003application,zhangpso}.
The earliest application of NCM to hyperspectral unmixing can be traced
back to \cite{stein2003application}, where a maximal likelihood estimation
(MLE) approach was presented for NCM endmember extraction. In \cite{eches2010bayesian,eches2010estimating},
priors (mainly uniform distributions) were imposed on endmembers and
abundances with sampling methods used to maximize the posterior. They
assumed the variance of each wavelength was independent and estimated
one parameter of variability for each endmember. In \cite{zare2010pce},
a Dirichlet prior distribution on the abundances was used. However,
in their model the covariance matrices of the endmembers are assumed
to be known instead of being unknown parameters to be estimated. In
\cite{zhangpso}, a more complicated NCM without the assumption of
independence of each wavelength was proposed. They maximize the posterior
using particle swarm optimization based expectation maximization.

However, there is little work on estimating the \emph{model uncertainty
}of the endmembers directly from the linear mixing model. That is,
given the pixel data and an estimated endmember set, the endmember
estimates may have residual uncertainty. For example, Figure~\ref{fig:uncertainty}
shows 3 possible estimated endmember sets on a synthetic dataset when
$B=2$. We can expect the endmember set $\mathbf{m}_{1},\mathbf{m}_{2},\mathbf{m}_{3}$
to have a small uncertainty since they fit the pixels very well. Allowing
them to move around may ruin the fitting. $\mathbf{m}_{1}^{\prime},\mathbf{m}_{2}^{\prime},\mathbf{m}_{3}^{\prime}$
are located within the pixels. They should have a large uncertainty
because they can move around more freely to better fit the pixels.
For $\mathbf{m}_{1}^{\prime\prime},\mathbf{m}_{2}^{\prime\prime},\mathbf{m}_{3}^{\prime\prime}$,
the uncertainty is ambiguous because they have already fitted the
pixels very well while they may also move around to some degree. Hence
we will not consider the uncertainty in this case in the present work.

\begin{figure}
\begin{centering}
\includegraphics[width=1\textwidth]{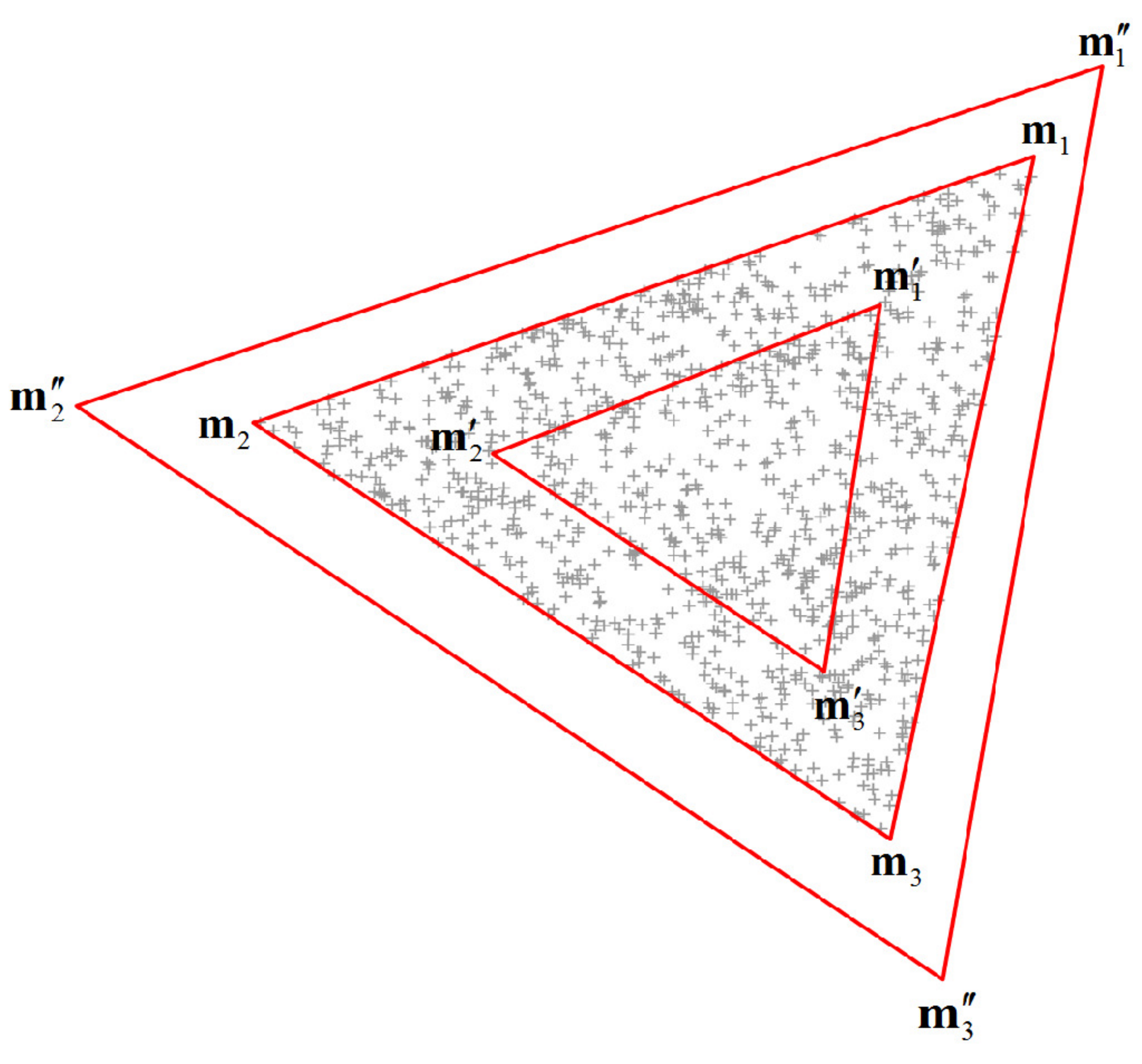}
\par\end{centering}

\caption{Model uncertainty of the estimated endmembers at different positions.
Intuitively, $\mathbf{m}_{1},\mathbf{m}_{2},\mathbf{m}_{3}$ should
have a small uncertainty while the uncertainty of $\mathbf{m}_{1}^{\prime},\mathbf{m}_{2}^{\prime},\mathbf{m}_{3}^{\prime}$
should be large. The uncertainty of $\mathbf{m}_{1}^{\prime\prime},\mathbf{m}_{2}^{\prime\prime},\mathbf{m}_{3}^{\prime\prime}$
will not be discussed here due to its ambiguity.}

\label{fig:uncertainty}
\end{figure}

The above intuition implies that the uncertainty may reflect the error
of endmembers. To show how this intuition formally works in NCM, assume
a simple case that an endmember $\mathbf{m}\in\mathbb{R}^{2}$ follows
a Gaussian distribution centered at $\mathbf{r}\in\mathbb{R}^{2}$
with covariance matrix $\boldsymbol{\Sigma}\in\mathbb{R}^{2\times2}$:
\[
p\left(\mathbf{m}\right)=\mathcal{N}\left(\mathbf{m}|\mathbf{r},\boldsymbol{\Sigma}\right).
\]
Suppose $\mathbf{m}$ is given and $\mathbf{r}$ has been estimated
with $\mathbf{m}\neq\mathbf{r}$. We want to find $\boldsymbol{\Sigma}$
using maximum likelihood estimation (MLE). Maximizing $p\left(\mathbf{m}\right)$
is equivalent to minimizing 
\[
-\log p(\mathbf{m})=\frac{1}{2}\text{log}\left|\boldsymbol{\Sigma}\right|+\frac{1}{2}\left(\mathbf{m}-\mathbf{r}\right)^{T}\boldsymbol{\Sigma}^{-1}\left(\mathbf{m}-\mathbf{r}\right).
\]
Let $\boldsymbol{\Sigma}=\mathbf{U}\text{diag}\left(\sigma_{1}^{2},\sigma_{2}^{2}\right)\mathbf{U}^{T}$,
$\sigma_{1}>0$, $\sigma_{2}>0$ be the eigendecomposition, then the
minimization problem above becomes
\[
\text{log}\sigma_{1}+\text{log}\sigma_{2}+\frac{1}{2}\sigma_{1}^{-2}z_{1}^{2}+\frac{1}{2}\sigma_{2}^{-2}z_{2}^{2}
\]
where $\mathbf{z}=\left(z_{1},z_{2}\right)^{T}=\mathbf{U}^{T}\left(\mathbf{m}-\mathbf{r}\right)$.
When the eigenvector in $\mathbf{U}$ is not perpendicular to $\mathbf{m}-\mathbf{r}$,
i.e. $z_{1}\neq0$, $z_{2}\neq0$, the minimization can be achieved
by setting the derivatives with respect to $\sigma_{1}$ and $\sigma_{2}$
to 0, which leads to 
\[
\sigma_{1}=\left|z_{1}\right|,\,\sigma_{2}=\left|z_{2}\right|.
\]
However, this is not the global minimum because if one eigenvector
in $\mathbf{U}$ is perpendicular to $\mathbf{m}-\mathbf{r}$ (the
other being parallel), say $z_{2}=0$, $z_{1}=\Vert\mathbf{m}-\mathbf{r}\Vert$,
$\sigma_{2}$ can be arbitrarily close to 0 such that $\text{log}\sigma_{2}$
goes to negative infinity. Assume $\sigma_{i}\ge\epsilon$ for a small
positive $\epsilon$ to make a solution exist, then the global minimum
lies at $\sigma_{1}=\Vert\mathbf{m}-\mathbf{r}\Vert$, $\sigma_{2}=\epsilon$.
Therefore, we can see that the MLE estimated matrix $\boldsymbol{\Sigma}$
has the square root of its largest eigenvalue equal to $\Vert\mathbf{m}-\mathbf{r}\Vert$
while its eigenvector is parallel to $\mathbf{m}-\mathbf{r}$. For
our formulation (\ref{eq:Y=00003DAM+N}), assume $\mathbf{M}$ follows
a Gaussian distribution with centers in $\mathbf{R}$. We then propose
a \emph{fundamental} question: 
\begin{itemize}
\item Given $\mathbf{Y}$, can we find the covariance matrices (uncertainty)
that measure the difference between the estimated endmembers $\mathbf{R}$
and the ground truth $\mathbf{M}$?
\end{itemize}
If the answer is yes, we have a measure to predict the error without
knowing the ground truth. This paper attempts to find such covariance
matrices.

The previous NCMs did not solve this problem. The covariance matrices
from the previous NCMs represent the \emph{endmember variability}
which arises from the assumption that the endmember set used for linearly
generating a pixel may vary per pixel due to atmospheric, environmental,
temporal factors and intrinsic variability in a material \cite{zare2014endmember}.
We explain the difference between these two concepts, uncertainty
and variability, here by first summarizing the previous NCMs. Suppose
the $j$th endmember follows a Gaussian distribution centered at $\mathbf{r}_{j}\in\mathbb{R}^{B}$
with covariance matrix $\boldsymbol{\Sigma}_{j}$:
\[
p(\mathbf{m}_{j})=\mathcal{N}\left(\mathbf{m}_{j}\vert\mathbf{r}_{j},\boldsymbol{\Sigma}_{j}\right).
\]
Assuming the endmembers to be independent, the random variable transformation
$\mathbf{y}_{i}=\mathbf{M}^{T}\boldsymbol{\alpha}_{i}$ for each pixel
suggests that the probability density function of $\mathbf{y}_{i}$
can be derived as
\[
p(\mathbf{y}_{i})=\mathcal{N}\left(\mathbf{y}_{i}\vert\mathbf{R}^{T}\boldsymbol{\alpha}_{i},\sum_{j=1}^{M}\alpha_{ij}^{2}\boldsymbol{\Sigma}_{j}\right)
\]
where $\mathbf{R}=\left[\mathbf{r}_{1},...,\mathbf{r}_{M}\right]^{T}$.
Then, NCM assumes the pixels are independent and obtains the density
function of $\mathbf{Y}$ as
\begin{equation}
p(\mathbf{Y})=\prod_{i=1}^{N}p(\mathbf{y}_{i})\label{eq:independentY}
\end{equation}
which is another Gaussian distribution with a \emph{block diagonal}
covariance matrix. The estimation of $\mathbf{r}_{j}$ and $\boldsymbol{\Sigma}_{j}$
is handled differently in different works.

To estimate the uncertainty however, we can not assume the pixels
to be independent. To see this, suppose $B=1$. Then we have $\mathbf{M}\in\mathbb{R}^{M}$,
$\mathbf{R}\in\mathbb{R}^{M}$, $\mathbf{Y}\in\mathbb{R}^{N}$ which
are all vectors and the covariance matrix in (\ref{eq:independentY})
becomes an $N$ by $N$ diagonal matrix. The independence of endmembers
suggests that the density of $\mathbf{M}$ is given by
\[
p(\mathbf{M})=\mathcal{N}\left(\mathbf{M}\vert\mathbf{R},\boldsymbol{\Sigma}\right)
\]
where $\boldsymbol{\Sigma}$ is an $M$ by $M$ diagonal matrix with
each element being the variance of each endmember. The random variable
transformation $\mathbf{Y}=\mathbf{A}\mathbf{M}$ indicates that the
density function of $\mathbf{Y}$ does not even exist. This is because
the domain of $p(\mathbf{M})$ ($\mathbb{R}^{M}$) is projected to
a subspace of dimension $M$ in $\mathbb{R}^{N}$, which has measure
0 (integrating $p(\mathbf{Y})$ would give value 0). The only way
to make the density function exist is to add noise, i.e., use equation
(\ref{eq:Y=00003DAM+N}). By assuming the noise to be Gaussian and
independent for each pixel, $p(\mathbf{N})=\mathcal{N}\left(\mathbf{N}\vert\mathbf{0},\mu^{2}\mathbf{I}_{N}\right)$,
we can see that the density function of $\mathbf{Y}$ becomes
\[
p(\mathbf{Y})=\mathcal{N}\left(\mathbf{Y}\vert\mathbf{A}\mathbf{R},\,\mathbf{A}\boldsymbol{\Sigma}\mathbf{A}^{T}+\mu^{2}\mathbf{I}_{N}\right)
\]
where the covariance matrix is not a diagonal matrix, which indicates
the pixels are \emph{not} independent. The general case with $B>1$
will be derived later.

In this paper, we solve the problem of estimating the model uncertainty
by proposing a \emph{spatial compositional model }(SCM) based on NCM
without assuming independence while utilizing spatial information
on the abundances. Compared to previous NCMs and methods with spatial
information, our method differs in the following aspects:
\begin{enumerate}
\item In contrast to the previous NCMs which assume pixel independence,
we estimate the full likelihood of the pixels (without the independence
assumption). Hence, we can obtain the endmember uncertainty that predicts
the error.
\item In the previous works, a uniform smoothness term is imposed to force
every two neighboring abundances to be similar \cite{eches2011enhancing,iordache2012total}.
In our work, the smoothness term varies locally according to pixel
information. Moreover, it is in a quadratic form which entails a simple
algorithm.
\item The previous works assume the covariance matrices of endmembers have
a simple form, e.g. diagonal (each wavelength is independent), which
neglects the correlation between wavelengths \cite{eches2010bayesian,eches2010estimating}.
Here we estimate the full covariance matrices and capture the correlations.
\end{enumerate}
The resulting model can be summarized by modeling the priors on abundances
based on the spatial information, the priors on endmembers based on
smoothness and principles of pairwise closeness, and transforming
the Gaussian probability functions to obtain the posterior which can
be maximized. The final minimization problem can be solved by a simple
and efficient optimization algorithm that not only provides the endmembers,
the abundances, but also the uncertainty. 

\emph{Notation}. Throughout the paper, $\text{SPD}\left(n\right)$
denotes the set of all $n$ by $n$ symmetric positive definite matrices.
We use the following notation for operations on a matrix $\mathbf{A}=\left[\mathbf{a}_{1},...,\mathbf{a}_{n}\right]$.
We use $\text{Tr}\left(\mathbf{A}\right)$, $\left|\mathbf{A}\right|$,
$\text{vec}\left(\mathbf{A}\right)$ to denote the trace, determinant,
vectorization of $\mathbf{A}$ respectively. The vectorization operator
is defined by concatenating its columns, $\text{vec}\left(\mathbf{A}\right)=\left[\mathbf{a}_{1}^{T},...,\mathbf{a}_{n}^{T}\right]^{T}$.
We use $[a_{ij}]$ to denote a matrix in which the element at the
$i$th row, $j$th column is $a_{ij}$. So the matrix $\left[\delta_{ij}a_{i}\right]$
is a diagonal matrix with diagonal element $a_{i}$ by defining $\delta_{ij}=1$
only when $i=j$ and 0 otherwise. We use $\mathbf{A}\ge0$ to denote
that $a_{ij}\ge0$ given $\mathbf{A}=\left[a_{ij}\right]$. The Kronecker
product between two matrices $\mathbf{A}$ and $\mathbf{B}$ is defined
by $\mathbf{A}\otimes\mathbf{B}=\left[a_{ij}\mathbf{B}\right]$. We
use $\Vert\mathbf{A}\Vert$, $\Vert\mathbf{A}\Vert_{F}$ as the operator
norm and Frobenius norm of $\mathbf{A}$ respectively. We use $\mathbf{I}_{N}$
for the $N$ by $N$ identity matrix and $\mathbf{1}_{N}$ as an $N$
by 1 vector consisting of all 1s.

\section{The Spatial Compositional Model}

\subsection{The hyperspectral image likelihood}

We are interested in determining the uncertainty of the extracted
endmembers. To achieve this, we  first model the density function
of $\mathbf{M}$, then use (\ref{eq:Y=00003DAM+N}) to perform a random
variable transformation to get the density function of $\mathbf{Y}$,
and finally maximize the posterior given $\mathbf{Y}$ to find the
parameters. Assuming that the endmember $\mathbf{m}_{j}$ follows
a multivariate Gaussian centered at $\mathbf{r}_{j}$ with covariance
matrix $\mathbf{\boldsymbol{\Sigma}}_{j}$, we have the conditional
probability density function of $\mathbf{m}_{j}$:
\[
p(\mathbf{m}_{j}\vert\mathbf{r}_{j},\boldsymbol{\Sigma}_{j})=\mathcal{N}\left(\mathbf{m}_{j}\vert\mathbf{r}_{j},\boldsymbol{\Sigma}_{j}\right).
\]
We can also assume that the endmembers are independent, which leads
to the conditional probability density function of the whole endmember
set to be the product of the independent components:
\begin{equation}
p(\mathbf{M}\vert\mathbf{R},\boldsymbol{\Theta})=\mathcal{N}\left(\text{vec}(\mathbf{M}^{T})\vert\text{vec}(\mathbf{R}^{T}),\,[\delta_{ij}\boldsymbol{\Sigma}_{j}]\right),\label{eq:MPdf}
\end{equation}
where $\boldsymbol{\Theta}=\left\{ \boldsymbol{\Sigma}_{j}\right\} $,
the covariance matrix is an $MB$ by $MB$ block diagonal matrix.

From the probability density function of the endmembers in (\ref{eq:MPdf}),
we can obtain the probability density function of $\mathbf{Y}$ from
the linear transformation in (\ref{eq:Y=00003DAM+N}). From straightforward
matrix algebra, we see that
\begin{eqnarray}
\text{vec}\left((\mathbf{AM})^{T}\right) & = & (\mathbf{A}\otimes\mathbf{I}_{B})\text{vec}(\mathbf{M}^{T}).\label{eq:M2AM}
\end{eqnarray}
Assuming that the noise $\mathbf{n}_{i}$ follows an independent zero
mean, $\mu^{2}$ variance Gaussian at each wavelength which is independent
at different locations, we have
\begin{equation}
p(\mathbf{N}\vert\mu)=\mathcal{N}\left(\text{vec}(\mathbf{N}^{T})\vert\mathbf{0},\,\mu^{2}\mathbf{I}_{NB}\right).\label{eq:NPdf}
\end{equation}
From the probability density functions in (\ref{eq:NPdf}), (\ref{eq:MPdf})
and the transformation in (\ref{eq:M2AM}), equation~(\ref{eq:Y=00003DAM+N})
indicates that the conditional probability density function of $\mathbf{Y}$
can be obtained as
\begin{equation}
p(\mathbf{Y}\vert\mathbf{R},\boldsymbol{\Theta},\mathbf{A},\mu)=\mathcal{N}\left(\text{vec}(\mathbf{Y}^{T})\vert\boldsymbol{\mu}_{\mathbf{Y}},\boldsymbol{\Sigma}_{\mathbf{Y}}\right),\label{eq:YCondPdf}
\end{equation}
where 
\begin{eqnarray*}
\boldsymbol{\mu}_{\mathbf{Y}} & = & (\mathbf{A}\otimes\mathbf{I}_{B})\text{vec}(\mathbf{R}^{T})\\
 & = & \text{vec}\left((\mathbf{AR})^{T}\right),
\end{eqnarray*}
\begin{eqnarray}
\boldsymbol{\Sigma}_{\mathbf{Y}} & = & (\mathbf{A}\otimes\mathbf{I}_{B})[\delta_{ij}\boldsymbol{\Sigma}_{j}](\mathbf{A}\otimes\mathbf{I}_{B})^{T}+\mu^{2}\mathbf{I}_{NB}\nonumber \\
 & = & \left[\delta_{ij}\mu^{2}\mathbf{I}_{B}+\sum_{k=1}^{M}\alpha_{ik}\alpha_{jk}\boldsymbol{\Sigma}_{k}\right].\label{eq:Def-SigmaY}
\end{eqnarray}
Note that the covariance matrix in (\ref{eq:Def-SigmaY}) is not a
block diagonal matrix, which means that the transformed rows in $\mathbf{Y}$
are not independent.

\subsection{Modeling the priors}

We model the prior probability density of $\mathbf{A}$ by assuming
that $\boldsymbol{\alpha}_{i}$ is a Markov random field (MRF). That
is, we treat the image grid as an undirected graph $\mathcal{G}=\left(\mathcal{V},\,\mathcal{E}\right)$
where $\mathcal{V}$ is the set of graph nodes and $\mathcal{E}$
is the set of edges. The density of the whole grid can be modeled
based on a potential function of the neighboring nodes. Suppose the
hyperspectral image is divided into different regions ($\mathcal{D}=\bigcup_{k=1}^{S}\Omega_{k}$,
$\Omega_{i}\bigcap\Omega_{j}=\emptyset$ when $i\neq j$) with the
pixels of a region showing similar reflectances, we have $S$ sets
of graph nodes $\mathcal{V}_{k}$, $k=1,...,S$. Then, the prior probability
density of $\mathbf{A}$ can be assumed to be in favor of smooth assignment
of $\boldsymbol{\alpha}_{i}$ to all the neighboring pixels within
a region, because they are more likely to be the same mixture of materials
and their abundances should be similar. 

Driven by this intuition, the prior probability density of $\mathbf{A}$
is modeled as
\begin{eqnarray}
p(\mathbf{A}) & \propto & \exp\left\{ -\frac{\beta_{1}}{4}\sum_{i=1}^{N}\sum_{j=1}^{N}w_{ij}\Vert\boldsymbol{\alpha}_{i}-\boldsymbol{\alpha}_{j}\Vert^{2}\right\} \nonumber \\
 & = & \exp\left\{ -\frac{\beta_{1}}{2}\text{Tr}\left(\mathbf{A}^{T}\mathbf{L}\mathbf{A}\right)\right\} ,\label{eq:APdf1}
\end{eqnarray}
where $w_{ij}$ controls the spatial intimacy between node $i$ and
node $j$, $\mathbf{L}=\left[\delta_{ij}\sum_{k}w_{ik}\right]-\left[w_{ij}\right]$
where $\mathbf{L}\in\mathbb{R}^{N\times N}$ is the well known symmetric
positive semidefinite \emph{graph Laplacian} matrix \cite{von2007tutorial}.
If we have a prior segmentation result, we can set $w_{ij}$ to be
1 when node $i$ and node $j$ are neighbors that belong to the same
region ${\cal V}_{k}$ and 0 otherwise. If we do not have a prior
segmentation, we can use 
\[
w_{ij}=e^{-\Vert\mathbf{y}_{i}-\mathbf{y}_{j}\Vert^{2}/2B\eta^{2}},
\]
when node $i$ and node $j$ are neighbors and 0 otherwise. From the
functional point of view, equation~(\ref{eq:APdf1}) can be seen
as trying to minimize $\sum_{j,k}\iint_{\Omega_{k}}\Vert\nabla\alpha_{j}\left(\mathbf{x}\right)\Vert^{2}d\mathbf{x}$
using a known segmentation. A similar graph regularizer is used in
\cite{cai2011graph,lu2013manifold}. However, in their work, only
pixel reflectances are used to construct the graph following the manifold
structure with no spatial information being incorporated.

In practice, a region may contain a pure material, which means the
abundance map for many pixels has its power concentrated on a single
component, e.g., $\alpha_{ij}=1$, $\alpha_{ik}=0$ for $k\neq j$.
This suggests that $\mathbf{A}$ should have a higher prior probability
for each $\boldsymbol{\alpha}_{i}$ being sparse. A common sparsity
promoting technique is to minimize the $L_{1}$ norm on $\boldsymbol{\alpha}_{i}$,
which is not applicable here due to the sum-to-one constraint. A previous
work uses the $L_{1/2}$ norm $\sum_{i,j}\alpha_{ij}^{1/2}$ to promote
sparsity \cite{qian2011hyperspectral}. However, the non-smooth objective
requires us to take subgradients which we would prefer to avoid. Here,
we introduce a quadratic form $\Vert\boldsymbol{\alpha}_{i}\Vert^{2}$,
which by itself is not sparsity promoting, but does have that effect
when maximized subject to the simplex constraint. Figure~\ref{fig:sparsity_promoting_effect}
shows the sparsity promoting effect if we want to maximize $\Vert\boldsymbol{\alpha}_{i}\Vert^{2}$
subject to the simplex constraint when $M=2$. For $M>2$, a similar
result can be achieved. Hence, we can add $\sum_{i}\Vert\boldsymbol{\alpha}_{i}\Vert^{2}$
to (\ref{eq:APdf1}) and have a prior probability defined as
\begin{eqnarray}
p\left(\mathbf{A}\right) & \propto & \exp\left\{ -\frac{\beta_{1}}{2}\text{Tr}\left(\mathbf{A}^{T}\mathbf{L}\mathbf{A}\right)+\frac{\beta_{2}}{2}\text{Tr}\left(\mathbf{A}^{T}\mathbf{A}\right)\right\} \nonumber \\
 & = & \exp\left\{ -\frac{\beta_{1}}{2}\text{Tr}\left(\mathbf{A}^{T}\mathbf{K}\mathbf{A}\right)\right\} ,\label{eq:APdf}
\end{eqnarray}
where $\mathbf{K}=\mathbf{L}-\frac{\beta_{2}}{\beta_{1}}\mathbf{I}_{N}$
if $\beta_{1}\neq0$.

\begin{figure}
\begin{centering}
\includegraphics[width=1\textwidth]{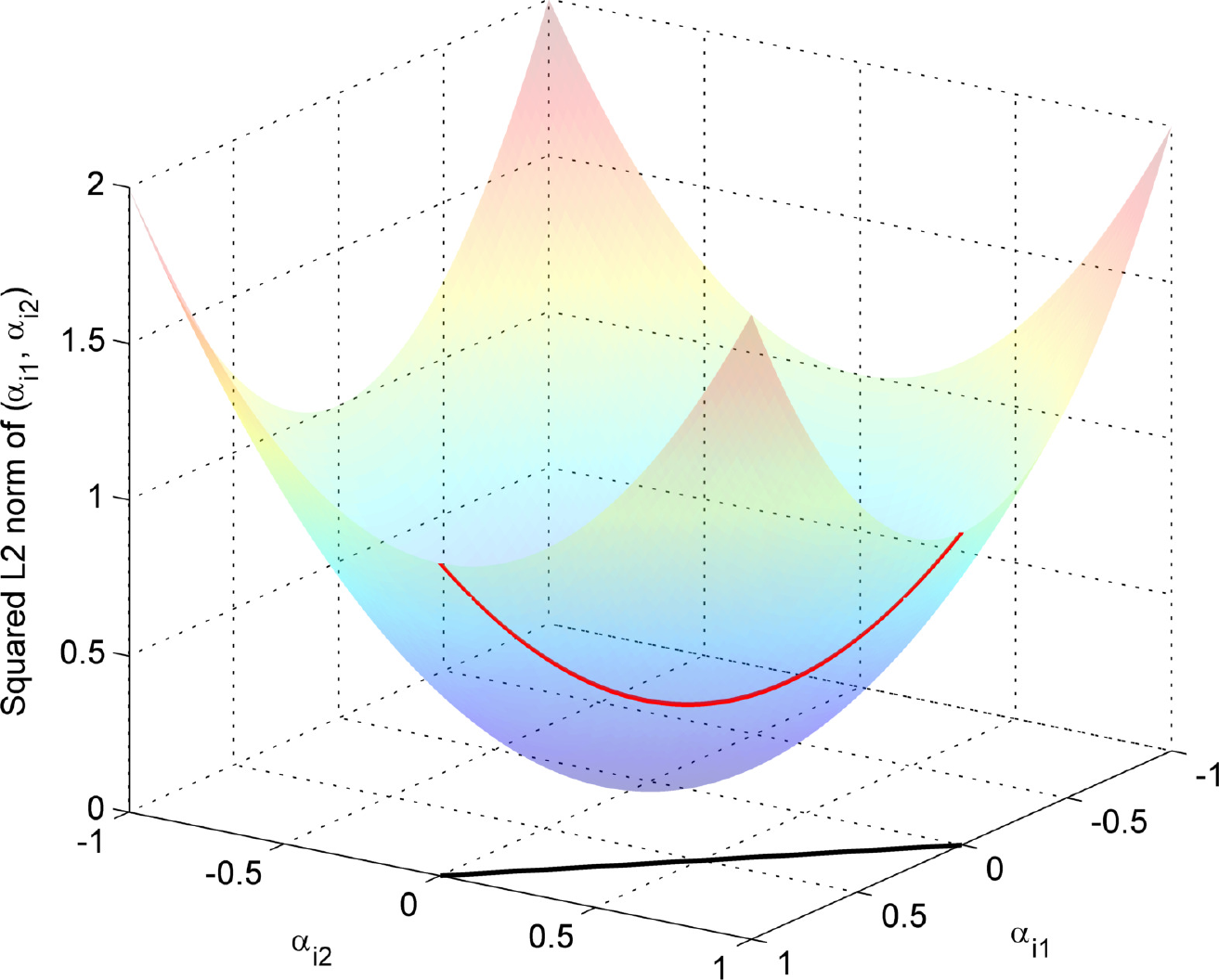}
\par\end{centering}

\caption{The sparsity promoting effect of maximizing $\Vert\boldsymbol{\alpha}_{i}\Vert^{2}$
subject to the simplex constraint when $M=2$. The black line segment
in the plane $z=0$ is the simplex constraint. The red line on the
paraboloid is the projected values of $\Vert\boldsymbol{\alpha}_{i}\Vert^{2}$
from the simplex. Maximizing $\Vert\boldsymbol{\alpha}_{i}\Vert^{2}$
will lead to solutions corresponding to the extreme ends of the simplex
(sparse solution).}

\label{fig:sparsity_promoting_effect}
\end{figure}

The parameters $\mathbf{r}_{j}$ can also be assumed to be drawn from
suitable prior distributions. From the analysis of Figure~\ref{fig:difficulties}
to obtain a unique solution, we assume that the endmembers should
tightly surround the mixed pixels. From the characteristics of function
representation, they should also be smooth. So we introduce two types
of proximity to model the density function of $\mathbf{R}$. The first
makes every two endmembers close to each other, which is also used
in \cite{berman2004ice,zare2013piecewise}. The second makes the adjacent
wavelengths have similar values for each endmember. This can be done
by using a density function on $\mathbf{R}$ as 
\begin{eqnarray}
p(\mathbf{R}) & \propto & \exp\left\{ -\frac{\rho_{1}}{4}\sum_{i=1}^{M}\sum_{j=1}^{M}u_{ij}\Vert\mathbf{r}_{i}-\mathbf{r}_{j}\Vert^{2}\right\} \times\nonumber \\
 &  & \text{exp}\left\{ -\frac{\rho_{2}}{4}\sum_{k=1}^{B}\sum_{l=1}^{B}v_{kl}\Vert\mathbf{r}^{k}-\mathbf{r}^{l}\Vert^{2}\right\} \label{eq:RPdf}
\end{eqnarray}
where $\mathbf{r}^{k}$ denotes the $k$th column of $\mathbf{R}$.
$u_{ij}=1$ for all $i$ and $j$ so the first term fulfills the first
sense of proximity. $v_{kl}$ is 1 when $\vert k-l\vert=1$ and 0
otherwise so it actually numerically minimizes $\sum_{j}\int\left(f_{j}^{\prime}(\lambda)\right)^{2}d\lambda$
and satisfies the second proximity. Similar to (\ref{eq:APdf}), (\ref{eq:RPdf})
can be written as
\begin{equation}
p(\mathbf{R})\propto\text{exp}\left\{ -\frac{\rho_{1}}{2}\text{Tr}\left(\mathbf{R}^{T}\mathbf{H}\mathbf{R}\right)-\frac{\rho_{2}}{2}\text{Tr}\left(\mathbf{R}\mathbf{G}\mathbf{R}^{T}\right)\right\} ,\label{eq:R1Pdf}
\end{equation}
where $\mathbf{H}\in\mathbb{R}^{M\times M}$ and $\mathbf{G}\in\mathbb{R}^{B\times B}$
are the corresponding Laplacian matrices ($\mathbf{H}$ has -1 everywhere
except for diagonals with value $M-1$, $\mathbf{G}$ has almost every
diagonal element as 2, except for 1 in $v_{11}$ and $v_{BB}$, and
-1 in the adjacent diagonals).

\subsection{Maximizing the posterior}

From the prior probability density in (\ref{eq:APdf}), (\ref{eq:R1Pdf})
and the conditional probability density in (\ref{eq:YCondPdf}), we
invoke Bayes' theorem to get the posterior probability density with
a view toward maximizing the posterior. Standard algebra yields

\[
p(\mathbf{R},\boldsymbol{\Theta},\mathbf{A},\mu\vert\mathbf{Y})\propto p(\mathbf{A})p(\mathbf{R})p(\mathbf{Y}\vert\mathbf{R},\boldsymbol{\Theta},\mathbf{A},\mu)
\]
where $p(\boldsymbol{\Theta})$, $p(\mu)$ and $p(\mathbf{Y})$ are
assumed to follow a uniform distribution. Maximizing $\log p\left(\mathbf{R},\boldsymbol{\Theta},\mathbf{A},\mu\vert\mathbf{Y}\right)$
is equivalent to minimizing 
\begin{align}
\mathcal{E}\left(\mathbf{R},\boldsymbol{\Theta},\mathbf{A},\mu\right)= & \text{vec}\left((\mathbf{Y}-\mathbf{AR})^{T}\right)^{T}\boldsymbol{\Sigma}_{\mathbf{Y}}^{-1}\text{vec}\left((\mathbf{Y}-\mathbf{AR})^{T}\right)+\log\vert\boldsymbol{\Sigma}_{\mathbf{Y}}\vert\nonumber \\
 & +\beta_{1}\text{Tr}\left(\mathbf{A}^{T}\mathbf{K}\mathbf{A}\right)+\rho_{1}\text{Tr}\left(\mathbf{R}^{T}\mathbf{H}\mathbf{R}\right)+\rho_{2}\text{Tr}\left(\mathbf{R}\mathbf{G}\mathbf{R}^{T}\right)\label{eq:E1}
\end{align}
where $\boldsymbol{\Sigma}_{\mathbf{Y}}$ is given in (\ref{eq:Def-SigmaY}).
Notice that the first term in (\ref{eq:E1}) involves inversion of
a large non-sparse $NB$ by $NB$ matrix, which is computationally
expensive. We now describe methods to reduce the complexity.

Using the \emph{Woodbury identity}, $\boldsymbol{\Sigma}_{\mathbf{Y}}^{-1}$
becomes 
\begin{eqnarray}
\boldsymbol{\Sigma}_{\mathbf{Y}}^{-1} & = & \left(\mu^{2}\mathbf{I}_{NB}+(\mathbf{A}\otimes\mathbf{I}_{B})[\delta_{ij}\boldsymbol{\Sigma}_{j}](\mathbf{A}\otimes\mathbf{I}_{B})^{T}\right)^{-1}\nonumber \\
 & = & \mu^{-2}\mathbf{I}_{NB}-\mu^{-2}\left(\mathbf{A}\otimes\mathbf{I}_{B}\right)\mathbf{Q}^{-1}\left(\mathbf{A}\otimes\mathbf{I}_{B}\right)^{T}\label{eq:sigmaInv}
\end{eqnarray}
where 
\begin{eqnarray}
\mathbf{Q} & = & \mu^{2}\left[\delta_{ij}\boldsymbol{\Sigma}_{j}\right]^{-1}+\left(\mathbf{A}\otimes\mathbf{I}_{B}\right)^{T}\left(\mathbf{A}\otimes\mathbf{I}_{B}\right)\nonumber \\
 & = & \mu^{2}\left[\delta_{ij}\boldsymbol{\Sigma}_{j}^{-1}\right]+\left(\mathbf{A}^{T}\otimes\mathbf{I}_{B}\right)\left(\mathbf{A}\otimes\mathbf{I}_{B}\right)\nonumber \\
 & = & \left[\delta_{ij}\mathbf{S}_{j}\right]+\mathbf{A}^{T}\mathbf{A}\otimes\mathbf{I}_{B}\label{eq:Def-T}
\end{eqnarray}
with $\mathbf{S}_{j}=\mu^{2}\boldsymbol{\Sigma}_{j}^{-1}$. Note that
$\mathbf{A}^{T}\mathbf{A}\otimes\mathbf{I}_{B}$ is a positive semidefinite
matrix and therefore $\mathbf{Q}\in\text{SPD}\left(MB\right)$ ($\boldsymbol{\Sigma}_{j}\in\text{SPD}\left(B\right)$).
Plugging (\ref{eq:sigmaInv}) into the first term of the objective
function leads to 
\begin{eqnarray}
 &  & \text{vec}\left((\mathbf{Y}-\mathbf{AR})^{T}\right)^{T}\boldsymbol{\Sigma}_{\mathbf{Y}}^{-1}\text{vec}\left((\mathbf{Y}-\mathbf{AR})^{T}\right)\nonumber \\
 & = & \mu^{-2}\Vert\mathbf{Y}-\mathbf{A}\mathbf{R}\Vert_{F}^{2}-\mu^{-2}\mathbf{z}^{T}\mathbf{Q}^{-1}\mathbf{z},\label{eq:E5-1}
\end{eqnarray}
where 
\begin{eqnarray*}
\mathbf{z} & = & \left(\mathbf{A}\otimes\mathbf{I}_{B}\right)^{T}\text{vec}\left(\left(\mathbf{Y}-\mathbf{A}\mathbf{R}\right)^{T}\right)\\
 & = & \left(\mathbf{A}^{T}\otimes\mathbf{I}_{B}\right)\text{vec}\left(\left(\mathbf{Y}-\mathbf{A}\mathbf{R}\right)^{T}\right)\\
 & = & \text{vec}\left(\left(\mathbf{Y}-\mathbf{A}\mathbf{R}\right)^{T}\mathbf{A}\right).
\end{eqnarray*}
From \emph{Sylvester's determinant theorem}, the logarithm term $\text{log}\left|\boldsymbol{\Sigma}_{\mathbf{Y}}\right|$
becomes
\begin{eqnarray}
 &  & \text{log}\left|\mu^{2}\mathbf{I}_{NB}+\left(\mathbf{A}\otimes\mathbf{I}_{B}\right)\left[\delta_{ij}\boldsymbol{\Sigma}_{j}\right]\left(\mathbf{A}\otimes\mathbf{I}_{B}\right)^{T}\right|\nonumber \\
 & = & \text{log}\mu^{2NB}\left|\mathbf{I}_{MB}+\mu^{-2}\left[\delta_{ij}\boldsymbol{\Sigma}_{j}\right]\left(\mathbf{A}\otimes\mathbf{I}_{B}\right)^{T}\left(\mathbf{A}\otimes\mathbf{I}_{B}\right)\right|\nonumber \\
 & = & \text{log}\mu^{2NB}\left|\mu^{-2}\left[\delta_{ij}\boldsymbol{\Sigma}_{j}\right]\right|\left|\mu^{2}\left[\delta_{ij}\boldsymbol{\Sigma}_{j}\right]^{-1}+\mathbf{A}^{T}\mathbf{A}\otimes\mathbf{I}_{B}\right|\nonumber \\
 & = & NB\text{log}\mu^{2}-\text{log}\left|\left[\delta_{ij}\mathbf{S}_{j}\right]\right|+\text{log}\left|\mathbf{Q}\right|.\label{eq:E5-2}
\end{eqnarray}

Combining the results in (\ref{eq:E5-1}) and (\ref{eq:E5-2}), and
letting $\gamma=\mu^{-2}$, minimizing (\ref{eq:E1}) becomes equivalent
to minimizing $\mathcal{E}_{1}\left(\mathbf{R},\mathbf{A},\gamma,\left\{ \mathbf{S}_{j}\right\} \right)$
as
\begin{eqnarray}
\mathcal{E}_{1}(\mathbf{R},\mathbf{A},\gamma,\left\{ \mathbf{S}_{j}\right\} ) & = & \gamma\Vert\mathbf{Y}-\mathbf{A}\mathbf{R}\Vert_{F}^{2}-\gamma\mathbf{z}^{T}\mathbf{Q}^{-1}\mathbf{z}+\text{log}\left|\mathbf{Q}\right|\nonumber \\
 &  & -\sum_{j=1}^{M}\text{log}\left|\mathbf{S}_{j}\right|-NB\text{log}\gamma+\beta_{1}\text{Tr}\left(\mathbf{A}^{T}\mathbf{K}\mathbf{A}\right)\nonumber \\
 &  & +\rho_{1}\text{Tr}\left(\mathbf{R}^{T}\mathbf{H}\mathbf{R}\right)+\rho_{2}\text{Tr}\left(\mathbf{R}\mathbf{G}\mathbf{R}^{T}\right)\label{eq:E6}
\end{eqnarray}
subject to
\begin{gather}
\mathbf{A}\geq0,\,\mathbf{A}\mathbf{1}_{M}=\mathbf{1}_{N},\,\mathbf{R}\ge0,\,\mathbf{S}_{j}\in\text{SPD}\left(B\right)\label{eq:cons}
\end{gather}
where $\mathbf{Q}=\left[\delta_{ij}\mathbf{S}_{j}\right]+\mathbf{A}^{T}\mathbf{A}\otimes\mathbf{I}_{B}$,
$\mathbf{z}=\text{vec}\left(\left(\mathbf{Y}-\mathbf{A}\mathbf{R}\right)^{T}\mathbf{A}\right)$.
Note that letting $\boldsymbol{\Sigma}_{j}\rightarrow\mathbf{0}$
(i.e. there is little endmember uncertainty) will result in $\mathbf{S}_{j}$
tending to infinity and $\gamma\mathbf{z}^{T}\mathbf{Q}^{-1}\mathbf{z}$
vanishing, resulting in $\text{log}\left|\mathbf{Q}\right|$ canceling
$\sum_{j=1}^{M}\text{log}\left|\mathbf{S}_{j}\right|$ as $\left[\delta_{ij}\mathbf{S}_{j}\right]$
dominates $\mathbf{Q}$. Thus the entire objective function reduces
to the widely used least squares objective.

\subsection{Optimizing the objective function}

The objective function (\ref{eq:E6}) is not convex. Given an initial
condition, we can use the block coordinate descent method to find
a suitable local minimum (please Section 2.7 in \cite{bertsekas1999nonlinear}).
That is, for $n=0,1,2,...$, $\mathbf{A},\,\mathbf{R}$ and $\gamma,\,\left\{ \mathbf{S}_{j}\right\} $
are alternately updated by 
\[
\mathbf{A}^{n+1},\,\mathbf{R}^{n+1}=\arg\min_{\mathbf{R},\mathbf{A}}\mathcal{E}_{1}\left(\mathbf{R},\mathbf{A},\gamma^{n},\left\{ \mathbf{S}_{j}^{n}\right\} \right),
\]
\[
\gamma^{n+1},\,\left\{ \mathbf{S}_{j}^{n+1}\right\} =\arg\min_{\gamma,\left\{ \mathbf{S}_{j}\right\} }\mathcal{E}_{1}\left(\mathbf{R}^{n+1},\mathbf{A}^{n+1},\gamma,\left\{ \mathbf{S}_{j}\right\} \right),
\]
subject to the constraints (\ref{eq:cons}). We show in the Appendix
that $\gamma\mathbf{z}^{T}\mathbf{Q}^{-1}\mathbf{z}$ and $\text{log}\left|\mathbf{Q}\right|-\sum_{j=1}^{M}\text{log}\left|\mathbf{S}_{j}\right|$
are positive and are small compared to $\gamma\Vert\mathbf{Y}-\mathbf{A}\mathbf{R}\Vert_{F}^{2}$
(also verified in the experiments to follow). Hence, when minimizing
$\mathcal{E}_{1}$ with respect to $\mathbf{A}$ and $\mathbf{R}$,
we can ignore these terms and instead minimize $\mathcal{E}_{2}\left(\mathbf{R},\mathbf{A}\right)$:
\begin{eqnarray}
\mathcal{E}_{2} & = & \Vert\mathbf{Y}-\mathbf{A}\mathbf{R}\Vert_{F}^{2}+\frac{\beta_{1}}{\gamma}\text{Tr}\left(\mathbf{A}^{T}\mathbf{K}\mathbf{A}\right)\nonumber \\
 &  & +\frac{\rho_{1}}{\gamma}\text{Tr}\left(\mathbf{R}^{T}\mathbf{H}\mathbf{R}\right)+\frac{\rho_{2}}{\gamma}\text{Tr}\left(\mathbf{R}\mathbf{G}\mathbf{R}^{T}\right).\label{eq:E7}
\end{eqnarray}
Also, in this case, $\gamma$ will not impact the solution of $\mathbf{A}$
and $\mathbf{R}$ because the ratio of parameters, e.g. $\beta_{1}/\gamma$,
will become the new parameters to tune. Assume when optimizing with
respect to $\mathbf{A},\,\mathbf{R}$, $\gamma$ is the optimal value.
The subsequent optimization with respect to $\gamma,\,\left\{ \mathbf{S}_{j}\right\} $
will not change the next iteration result of $\mathbf{A},\,\mathbf{R}$,
which in turn keeps $\gamma,\,\left\{ \mathbf{S}_{j}\right\} $ unchanged.
So the block coordinate descent becomes a simple two step algorithm
where the first step minimizes (\ref{eq:E7}) with respect to $\mathbf{A},\,\mathbf{R}$
and the second step minimizes (\ref{eq:E6}) with respect to $\gamma,\,\left\{ \mathbf{S}_{j}\right\} $
given the obtained $\mathbf{A},\,\mathbf{R}$. Note that again both
of them are optimizations over convex sets ($\mathbf{A}$ is restricted
to the Cartesian product of simplices, $\mathbf{S}_{j}$ is restricted
to the convex cone of positive definite matrices) and so gradient
projection methods can be used to solve these kind of problems (please
see Section 2.3 in \cite{bertsekas1999nonlinear}).

Though the objective function (\ref{eq:E7}) is not convex, it is
convex with respect to either $\mathbf{A}$ or $\mathbf{R}$ (e.g.
it can be written as a quadratic function with respect to $\mathbf{A}$:
$\frac{1}{2}\mathbf{x}^{T}\mathbf{Q}\mathbf{x}+\mathbf{b}^{T}\mathbf{x}$
where $\mathbf{x}=\text{vec}\left(\mathbf{A}\right)$, $\mathbf{Q}=\mathbf{R}\mathbf{R}^{T}\otimes\mathbf{I}_{N}+\frac{\beta_{1}}{\gamma}\mathbf{I}_{M}\otimes\mathbf{K}$,
$\mathbf{b}=-\text{vec}\left(\mathbf{Y}\mathbf{R}^{T}\right)$). We
can alternately update $\mathbf{A}$ and $\mathbf{R}$ to reduce the
energy. Taking derivatives of (\ref{eq:E7}) with respect to $\mathbf{A}$,
we have
\begin{equation}
\frac{\partial\mathcal{E}_{2}}{\partial\mathbf{A}}=2\left(-\mathbf{Y}\mathbf{R}^{T}+\mathbf{A}\mathbf{R}\mathbf{R}^{T}+\frac{\beta_{1}}{\gamma}\mathbf{KA}\right).\label{eq:derA}
\end{equation}
The gradient projection method sets the value of the next iteration,
$\mathbf{A}^{n+1}$, to be the projected value of the steepest descent
\begin{equation}
\mathbf{A}^{n+1}=\phi\left(\mathbf{A}^{n}-\frac{\tau^{n}}{2}\frac{\partial\mathcal{E}_{2}}{\partial\mathbf{A}}\left(\mathbf{R}^{n},\mathbf{A}^{n}\right)\right),\label{eq:optimalA}
\end{equation}
where
\[
\phi:\,\mathbf{X}\mapsto\arg\min_{\mathbf{Y}\in\mathbb{R}^{N\times M}}\Vert\mathbf{X}-\mathbf{Y}\Vert_{F}^{2}\,\text{s.t. }\mathbf{Y}\ge0,\,\mathbf{Y}\mathbf{1}_{M}=\mathbf{1}_{N}
\]
projects a matrix to the nearest matrix that satisfies the simplex
constraint (e.g. we use the algorithm of Figure 1 in \cite{duchi2008efficient}).
$\tau^{n}>0$ is the step size and is set by 1D minimization or the
familiar Armijo rule. It is shown that the sequence generated by (\ref{eq:optimalA})
is gradient related, i.e. $\left\langle \frac{\partial\mathcal{E}_{2}}{\partial\mathbf{A}}\left(\mathbf{R}^{n},\mathbf{A}^{n}\right),\mathbf{A}^{n+1}-\mathbf{A}^{n}\right\rangle <0$
(Proposition 2.3.1 in \cite{bertsekas1999nonlinear}), which leads
to a stationary point given proper step sizes $\tau^{n}$ such as
the exact line minimization of \cite{hager2004gradient},
\[
\tau^{n}=\arg\min_{\tau\ge0}\mathcal{E}_{2}\left(\mathbf{R}^{n},\phi\left(\mathbf{A}^{n}-\frac{\tau}{2}\frac{\partial\mathcal{E}_{2}}{\partial\mathbf{A}}\left(\mathbf{R}^{n},\mathbf{A}^{n}\right)\right)\right).
\]
Numerically, we can use adaptive step sizes that start with a small
step and gradually increase it by an order of magnitude until $\mathcal{E}_{2}$
starts increasing. Similar gradient descent methods were proposed
in \cite{guan2011manifold,lin2007projected}, and it is shown that
such methods have a faster convergence rate than those based on multiplicative
update rules \cite{cai2011graph,lu2013manifold}.

Once we have updated $\mathbf{A}$, we can update $\mathbf{R}$ by
finding a new value that reduces (\ref{eq:E7}). A gradient projection
method can also be used for $\mathbf{R}$ because of the positivity
constraint. However, the introduction of spatial smoothness and the
sparsity promoting term, $\text{Tr}\left(\mathbf{A}^{T}\mathbf{K}\mathbf{A}\right)$,
along with the pairwise closeness term actually make $\mathbf{R}$
seldom negative even when just using a closed form solution. Taking
derivatives of (\ref{eq:E7}) with respect to $\mathbf{R}$, we have
\[
\frac{\partial\mathcal{E}_{2}}{\partial\mathbf{R}}=2\left(\mathbf{A}^{T}\mathbf{A}\mathbf{R}-\mathbf{A}^{T}\mathbf{Y}+\frac{\rho_{1}}{\gamma}\mathbf{HR}+\frac{\rho_{2}}{\gamma}\mathbf{RG}\right).
\]
Letting $\frac{\partial\mathcal{E}_{2}}{\partial\mathbf{R}}=0$, we
obtain a closed form solution for $\mathbf{R}$ that ignores the positivity
constraint, 
\begin{equation}
\left(\mathbf{A}^{T}\mathbf{A}+\frac{\rho_{1}}{\gamma}\mathbf{H}\right)\mathbf{R}+\frac{\rho_{2}}{\gamma}\mathbf{RG}=\mathbf{A}^{T}\mathbf{Y}.\label{eq:optimalR}
\end{equation}
Equation (\ref{eq:optimalR}) is called a \emph{Sylvester equation}
in control theory which is normally solved by performing a Schur decomposition
of the two matrices before and after $\mathbf{R}$ and back substitution
of the resulting equations \cite{bartels1972solution}. Given an initial
condition, we can alternately update $\mathbf{A}$ and $\mathbf{R}$
based on (\ref{eq:optimalA}) and (\ref{eq:optimalR}). The details
are given in the first two steps in Algorithm~\ref{Algo1}. Since
at each step the energy is lowered, the algorithm will lead to a local
minimum. 

Given the estimated endmembers and abundances, we can find $\gamma,\,\left\{ \mathbf{S}_{j}\right\} $
(hence $\mu$ and $\boldsymbol{\Sigma}_{j}$) similarly. Taking derivatives
of (\ref{eq:E6}) with respect to $\gamma$ and setting it to zero,
we have 
\begin{equation}
\gamma^{-1}=\frac{1}{NB}\left\{ \Vert\mathbf{Y}-\mathbf{A}\mathbf{R}\Vert_{F}^{2}-\mathbf{z}^{T}\mathbf{Q}^{-1}\mathbf{z}\right\} .\label{eq:optimalGamma1}
\end{equation}
Note that the right hand side of (\ref{eq:optimalGamma1}) is always
greater than 0 from (\ref{eq:E5-1}). Using the chain rule in matrix
form to take derivatives of (\ref{eq:E6}) with respect to $\mathbf{S}_{j}$,
we have 
\begin{equation}
\frac{\partial\mathcal{E}_{1}}{\partial\mathbf{S}_{j}}=\left(\gamma\mathbf{Q}^{-1}\mathbf{z}\mathbf{z}^{T}\mathbf{Q}^{-1}\right)_{j}-\mathbf{S}_{j}^{-1}+\left(\mathbf{Q}^{-1}\right)_{j},\label{eq:derSigma_j}
\end{equation}
where $\left(\cdot\right)_{j}$ denotes the extraction of the $j$th
diagonal $B$ by $B$ block of the $MB$ by $MB$ matrix. Hence, we
can alternately update $\gamma$ and $\mathbf{S}_{j}$ to minimize
(\ref{eq:E6}) keeping $\mathbf{A},\,\mathbf{R}$ fixed. For updating
$\mathbf{S}_{j}$, a similar gradient projection method as (\ref{eq:optimalA})
can be used, where the projection onto the set of positive definite
matrices is obtained by truncating the eigenvalues\cite{absil2012projection}.
The details are given in Step 3 of Algorithm~\ref{Algo1}.

\begin{algorithm}
\caption{The implementation of SCM}

Input: $\mathbf{Y}=\left[\mathbf{y}_{1},...,\mathbf{y}_{N}\right]^{T}$,
$M$, $\eta$, $\beta_{1}^{\prime}$, $\beta_{2}^{\prime}$, $\rho_{1}^{\prime}$,
$\rho_{2}^{\prime}$, $\sigma_{0}$, $\sigma_{\text{max}}$.
\begin{itemize}
\item Step 1: Initialize $\frac{\beta_{1}}{\gamma}=\frac{B}{M}\beta_{1}^{\prime}$,
$\frac{\beta_{2}}{\gamma}=\frac{B}{M}\beta_{2}^{\prime}$, $\frac{\rho_{1}}{\gamma}=\frac{N}{M^{2}}\rho_{1}^{\prime}$,
$\frac{\rho_{2}}{\gamma}=\frac{N}{M}\rho_{2}^{\prime}$\textsuperscript{1}.
Construct the Laplacian matrices $\mathbf{L}$, $\mathbf{H}$ and
$\mathbf{G}$. 
\item Step 2: Initialize $\mathbf{R}$ to be the centers of $M$ clusters
of $\mathbf{Y}$ by K-means\textsuperscript{2}. Initialize $\mathbf{A}=\phi\left(\mathbf{Y}\mathbf{R}^{T}\left(\mathbf{R}\mathbf{R}^{T}+\epsilon\mathbf{I}_{M}\right)^{-1}\right)$,
where $\phi\left(\mathbf{A}\right):\mathbf{A}\mapsto\left[\max\left(\alpha_{ij}-\theta_{i},0\right)\right]$
where $\theta_{i}=\frac{1}{K_{i}}\left(\sum_{k=1}^{K_{i}}\alpha_{ik}^{\prime}-1\right)$,
$\alpha_{i1}^{\prime}\geq...\geq\alpha_{iM}^{\prime}$ are sorted
$\alpha_{i1},...,\alpha_{iM}$, $K_{i}$ is the largest $k$ such
that $\alpha_{ik}^{\prime}-\frac{1}{k}\left(\sum_{l=1}^{k}\alpha_{il}^{\prime}-1\right)>0$.
Solve $\mathbf{A},\,\mathbf{R}$ by repeating the following two steps
until convergence. 

\begin{itemize}
\item Update $\mathbf{A}$ by $\phi_{A}\left(\tau\right)=\phi\left(\mathbf{A}-\frac{\tau}{2}\frac{\partial\mathcal{E}_{2}}{\partial\mathbf{A}}\right)$,
where $\frac{\partial\mathcal{E}_{2}}{\partial\mathbf{A}}$ is given
in (\ref{eq:derA}). If $\mathcal{E}_{2}\left(\mathbf{R},\phi_{A}\left(\tau_{\epsilon}\right)\right)<\mathcal{E}_{2}\left(\mathbf{R},\mathbf{A}\right)$,
$\tau$ attempts $\tau_{\epsilon}10^{i}$, $i=0,1,2,...$ until $\mathcal{E}_{2}\left(\mathbf{R},\phi_{A}\left(\tau_{\epsilon}10^{i+1}\right)\right)\ge\mathcal{E}_{2}\left(\mathbf{R},\phi_{A}\left(\tau_{\epsilon}10^{i}\right)\right)$,
otherwise set $\tau=0$.
\item Update $\mathbf{R}$ by solving (\ref{eq:optimalR})\textsuperscript{3}.
\end{itemize}
\item Step 3: Initialize $\gamma^{-1}=\frac{1}{NB}\Vert\mathbf{Y}-\mathbf{A}\mathbf{R}\Vert_{F}^{2}$,
$\boldsymbol{\Sigma}_{j}=\sigma_{0}^{2}\mathbf{I}_{B}$. Define $\psi\left(\mathbf{X}\right):\mathbf{X}\mapsto\mathbf{U}\left[\delta_{ij}\text{max}\left(\lambda_{i},1/\gamma\sigma_{\text{max}}^{2}\right)\right]\mathbf{U}^{T}$
where $\mathbf{X}=\mathbf{U}\left[\delta_{ij}\lambda_{i}\right]\mathbf{U}^{T}$
is the eigendecomposition. Solve $\gamma,\,\left\{ \mathbf{S}_{j}\right\} $
by repeating the following two steps until convergence.

\begin{itemize}
\item Update $\mathbf{S}_{j}$ by $\psi_{j}\left(\tau\right)=\psi\left(\mathbf{S}_{j}-\tau\frac{\partial\mathcal{E}_{1}}{\partial\mathbf{S}_{j}}\right)$
for $j=1,...,M$, where $\frac{\partial\mathcal{E}_{1}}{\partial\mathbf{S}_{j}}$
is given in (\ref{eq:derSigma_j}). The step size $\tau$ is determined
similar to step 2.
\item Update $\gamma$ by (\ref{eq:optimalGamma1}).
\end{itemize}
\end{itemize}
Output: $\mathbf{A}$, $\mathbf{R}$, $\mu=\gamma^{-1/2}$, $\boldsymbol{\Sigma}_{j}=\mu^{2}\mathbf{S}_{j}^{-1}$.

\label{Algo1}
\end{algorithm}

\emph{Remark 1}. The choice of free parameters should be invariant
with respect to the changing magnitude of each term in (\ref{eq:E7})
with different $N$, $M$ and $B$. For example, the first term in
(\ref{eq:E7}) has a magnitude of $NB$. From the banded diagonal
nature of $\left[w_{ij}\right]$ in (\ref{eq:APdf}), $\text{Tr}\left(\mathbf{A}^{T}\mathbf{L}\mathbf{A}\right)$
has a magnitude of $NM$. So the parameter $\beta_{1}$ should have
a magnitude of $\beta_{1}/\gamma=\beta_{1}^{\prime}B/M$. Similarly,
the parameters $\beta_{2}$, $\rho_{1}$ and $\rho_{2}$ should have
magnitudes according to $\beta_{2}/\gamma=\beta_{2}^{\prime}B/M$,
$\rho_{1}/\gamma=\rho_{1}^{\prime}N/M^{2}$, $\rho_{2}/\gamma=\rho_{2}^{\prime}N/M$.

\emph{Remark 2}. The initial endmembers are important in endmember
estimation. Randomly picking pixels and fast algorithms such as VCA
can provide an initial estimate. We find that K-means works well in
practical applications. This could be due to the fact that K-means
can pre-segment the image to obtain the mean values of different regions.

\emph{Remark 3}. We can resort to the classical Krylov subspace method
to solve the transposed version of (\ref{eq:optimalR}) more efficiently
\cite{el2002block}. That is, for $\mathbf{A}\mathbf{X}+\mathbf{X}\mathbf{B}=\mathbf{C}$
where $\mathbf{A}\in\mathbb{R}^{B\times B}$, $\mathbf{B}\in\mathbb{R}^{M\times M}$,
$\mathbf{X},\,\mathbf{C}\in\mathbb{R}^{B\times M}$, $B\gg M$, $\mathbf{A}$
is sparse and $\mathbf{B}$ is symmetric, the eigendecomposition $\mathbf{B}=\mathbf{U}\boldsymbol{\Lambda}\mathbf{U}^{T}$
($\boldsymbol{\Lambda}=\left[\delta_{ij}\lambda_{i}\right]$ consists
of eigenvalues) shows that it is equivalent to $\mathbf{A}\mathbf{Y}+\mathbf{Y}\boldsymbol{\Lambda}=\mathbf{D}$
where $\mathbf{Y}=\mathbf{X}\mathbf{U}$, $\mathbf{D}=\mathbf{C}\mathbf{U}$.
Let $\mathbf{Y}=\left[\mathbf{y}_{1},...,\mathbf{y}_{M}\right]$,
$\mathbf{D}=\left[\mathbf{d}_{1},...,\mathbf{d}_{M}\right]$, we have
a linear system of equations $\left(\mathbf{A}+\lambda_{i}\mathbf{I}\right)\mathbf{y}_{i}=\mathbf{d}_{i}$
for each column, where $\mathbf{y}_{i}$ can be solved independently
and efficiently since $\mathbf{A}$ is sparse. Then, $\mathbf{X}$
can be recovered using $\mathbf{X}=\mathbf{Y}\mathbf{U}^{T}$.

\section{Results}

In the experiments, all the algorithms were implemented in MATLAB\textsuperscript{\textregistered}.
For endmember extraction, we compared the SCM algorithm with NCM and
PCOMMEND \cite{zare2013piecewise}, where NCM was implemented as SCM
with $\beta_{1}^{\prime}=0$, $\beta_{2}^{\prime}=0$, $\rho_{2}^{\prime}=0$.
The parameters of SCM have fixed $\eta=0.05$, $\beta_{1}^{\prime}=0.01$,
$\rho_{2}^{\prime}=0$, $\sigma_{0}=0.1$, $\sigma_{\text{max}}=1$
for all the cases. The parameters of PCOMMEND were tuned to give the
best result in each case. Throughout the experiments, we use the mean
of absolute difference as the error, i.e., $\frac{1}{NM}\sum_{i,j}\vert\alpha_{ij}-\alpha_{ij}^{\prime}\vert$
for error of abundances, $\frac{1}{MB}\sum_{i,j}\vert m_{ij}-m_{ij}^{\prime}\vert$
for error of endmembers. Because the endmembers from algorithms may
have a different permutation from the ground truth endmembers, we
permuted the results from algorithms to calculate the error.

For measuring and visualizing the uncertainty from $\left\{ \boldsymbol{\Sigma}_{j}\right\} $,
recall that the covariance matrix of a Gaussian distribution determines
the shape of the distribution, i.e. the eigenvectors are the directions
of the variation patterns while the eigenvalues are the variances
of the projected (onto the eigenvectors) 1D data points. The uncertainty
can be measured by the largest eigenvalue and its corresponding eigenvector.
We use the square root of the largest eigenvalue, $\sigma$, as the
\emph{uncertainty amount} since it corresponds to the standard deviation.
Then, the corresponding eigenvector (normalized), $\mathbf{u}$, can
be viewed as the \emph{uncertainty direction}. The \emph{uncertainty
range} can be visualized by the estimated endmember $\mathbf{r}$
plus (minus) twice the uncertainty direction with uncertainty amount,
i.e. $\mathbf{r}\pm2\sigma\mathbf{u}$.

\subsection{Synthetic images}

We first test SCM on synthetic images generated from the true material
spectra in the Aster spectral library \cite{baldridge2009aster}.
We picked 2 rocks (limestone, basalt), 2 man-made materials (concrete,
asphalt) in the experiments. The spectra of these endmembers are shown
in Figure~\ref{fig:endmemberSignature}. The wavelength of these
materials ranges from 0.4$\mu$m to 14$\mu$m. For each material,
the reflectance of this range is re-sampled into 200 values.

\begin{figure}
\begin{centering}
\includegraphics[width=1\textwidth]{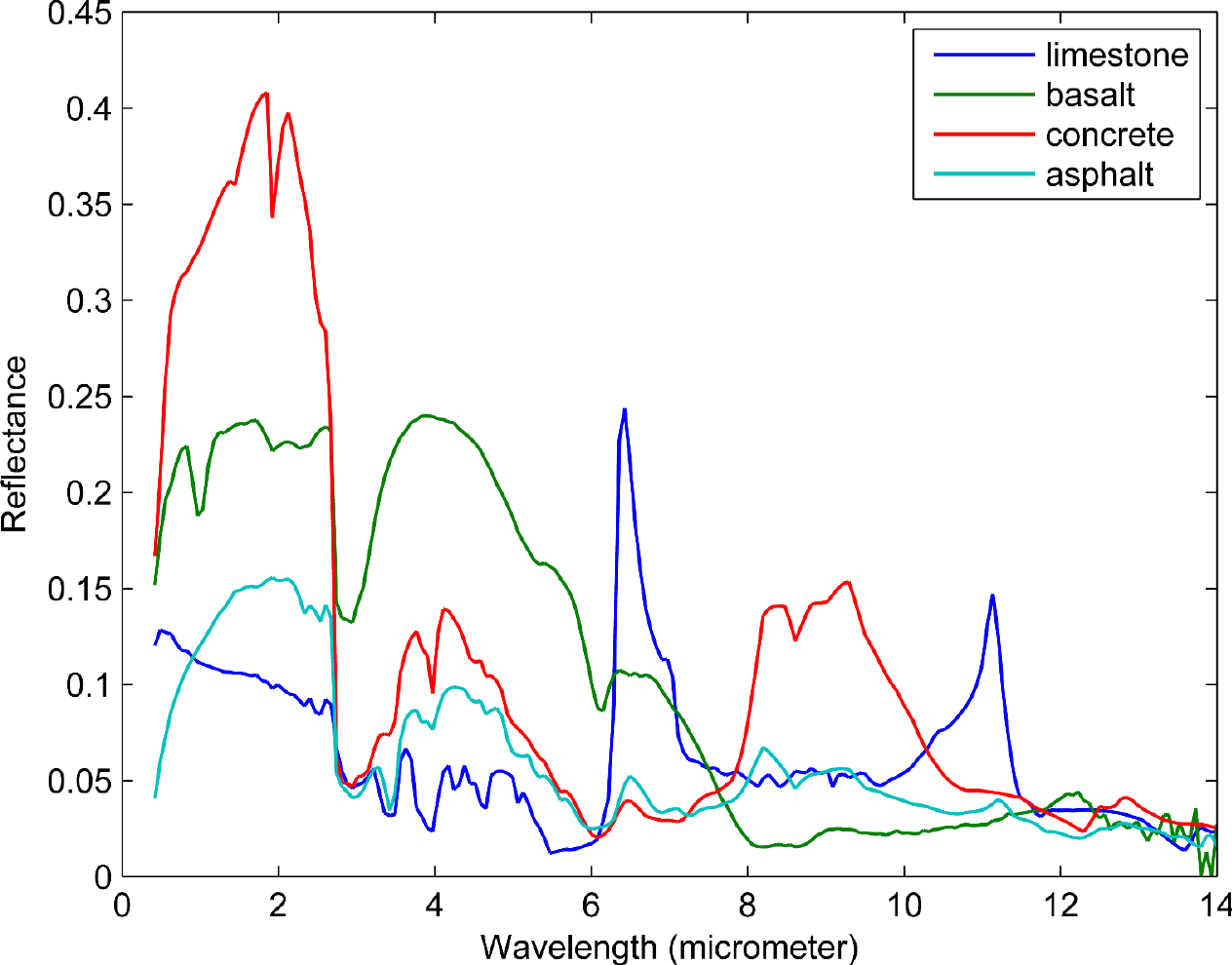}
\par\end{centering}

\caption{Spectral signatures of the 4 endmembers used to generate synthetic
images.}

\label{fig:endmemberSignature}
\end{figure}

A set of synthetic images of size 40 by 40 were generated using the
4 endmembers with different noise levels. For each image, the domain
is divided into 4 rectangular regions, where each region contains
a pure material. Hence, the abundance maps contain 1 corresponding
to the pure material at each pixel and 0 for the other materials.
Then, each abundance map is convolved with an isotropic 2D Gaussian
filter such that the boundary between regions is blurred and the nearby
pixels contain mixed materials. At the end, an additive noise with
mean zero and standard deviation $\sigma_{Y}$ is added to the image.
We conducted experiments on these images to verify the ability to
find endmembers and the ability to estimate the uncertainty.

For endmember extraction, we compared SCM, NCM, and PCOMMEND based
on 5 levels of signal-to-noise ratio (SNR), from 20dB ($\sigma_{Y}\approx0.01$)
to 60dB ($\sigma_{Y}\approx0.0001$). 20 random images were generated
in each case such that the average error can be calculated. The parameters
of SCM were $\beta_{2}^{\prime}=0$, $\rho_{1}^{\prime}=0.005$. Figure~\ref{fig:synthetic_error}
shows the errors of all the algorithms. From the plots, we can see
that SCM has lower errors than NCM and PCOMMEND for all the noise
cases, with respect to both endmembers and abundances. Figure~\ref{fig: synthetic_abundances}
shows the abundance maps from these algorithms for a noisy synthetic
image. We can see that the abundance maps of NCM and PCOMMEND present
more fuzzy abundances within a pure material region due to the noise.
Meanwhile, SCM presents consistent abundances within such a region.

\begin{figure}
\begin{centering}
\includegraphics[width=1\textwidth]{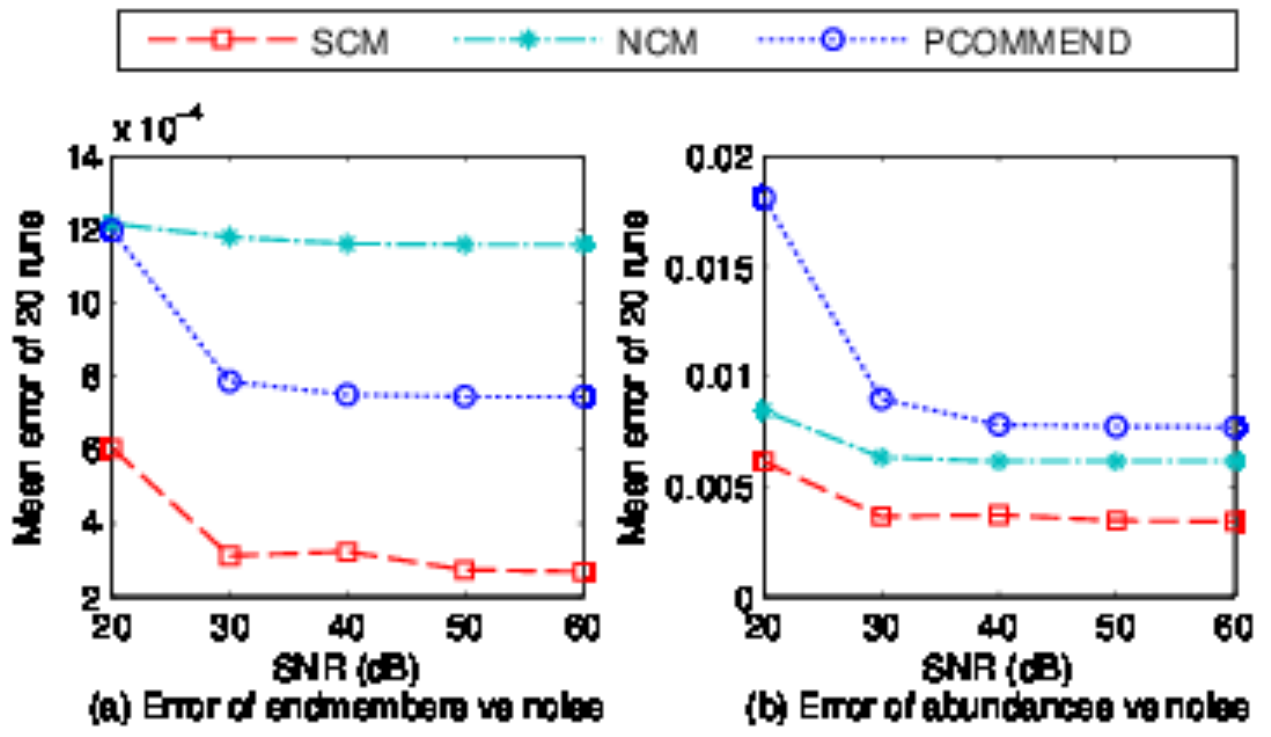}
\par\end{centering}

\caption{Error of endmembers and abundances based on the synthetic imagesfor
all the algorithms.}

\label{fig:synthetic_error}
\end{figure}

\begin{figure}
\begin{centering}
\includegraphics[width=1\textwidth]{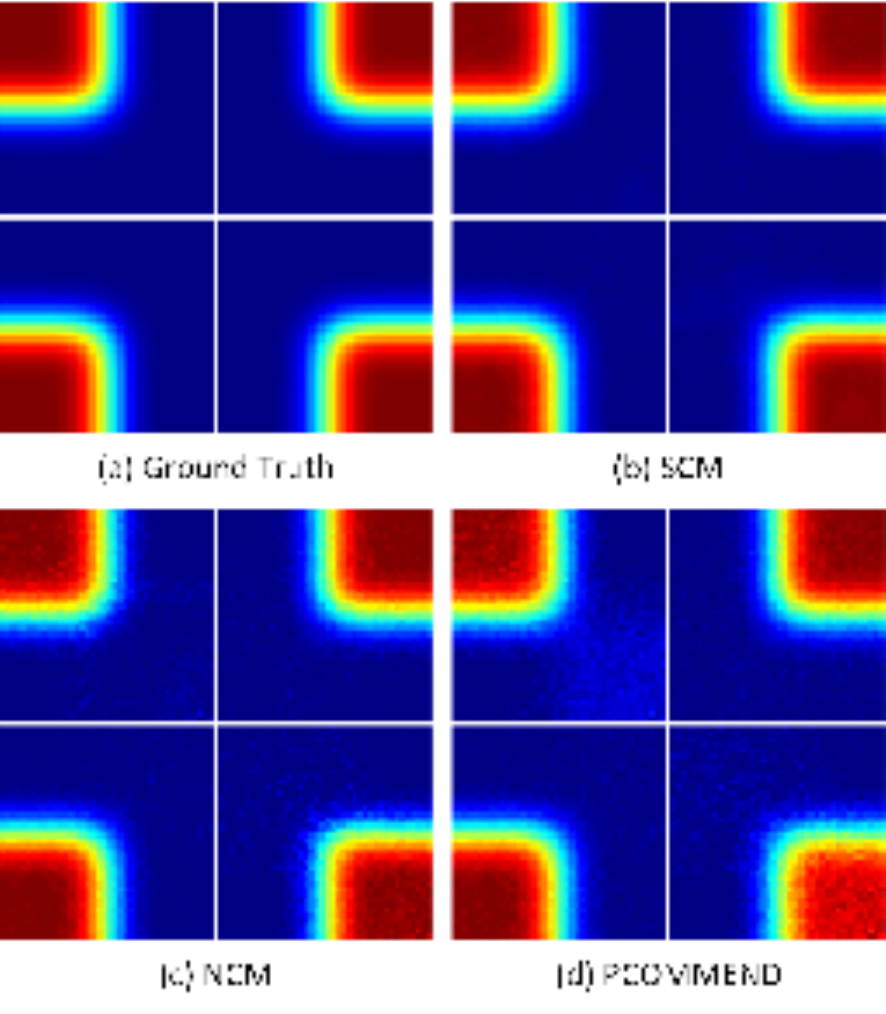}
\par\end{centering}

\caption{Abundance maps from the ground truth (a), SCM (b), NCM (c) and PCOMMEND
(d) for a noisy synthetic image with SNR 20dB.}

\label{fig: synthetic_abundances}
\end{figure}

For uncertainty estimation, we compared the uncertainties based on
different estimated endmembers for the same material in a synthetic
image with SNR 40dB. To achieve this, we changed the value of $\rho_{1}^{\prime}$
from large to small gradually. This causes the location of the estimated
endmembers to change from being close together inside the pixel cloud
to sparsely scattered outside the pixel cloud. We pick the uncertainty
amount of limestone to represent the whole uncertainty. Figure~\ref{fig:rho1_uncer}
shows this value along with the error of endmembers versus decreasing
$\rho_{1}^{\prime}$. The error of endmembers has its minimum in the
middle between $10^{-2}$ and $10^{-3}$. Interestingly, this is also
the place where the uncertainty amount starts to decrease to a stable
value. This corresponds to the intuition that when the endmembers
are outside the pixel cloud, all the pixels can be well represented
by the endmembers thus we have a low uncertainty, while when the endmembers
are inside the pixel cloud, the more they are closely packed together,
more the uncertainty as more pixels are beyond their representation
capabilities. Recalling our fundamental question about error prediction,
the result here implies that we are capable of estimating the uncertainty.

\begin{figure}
\begin{centering}
\includegraphics[width=1\textwidth]{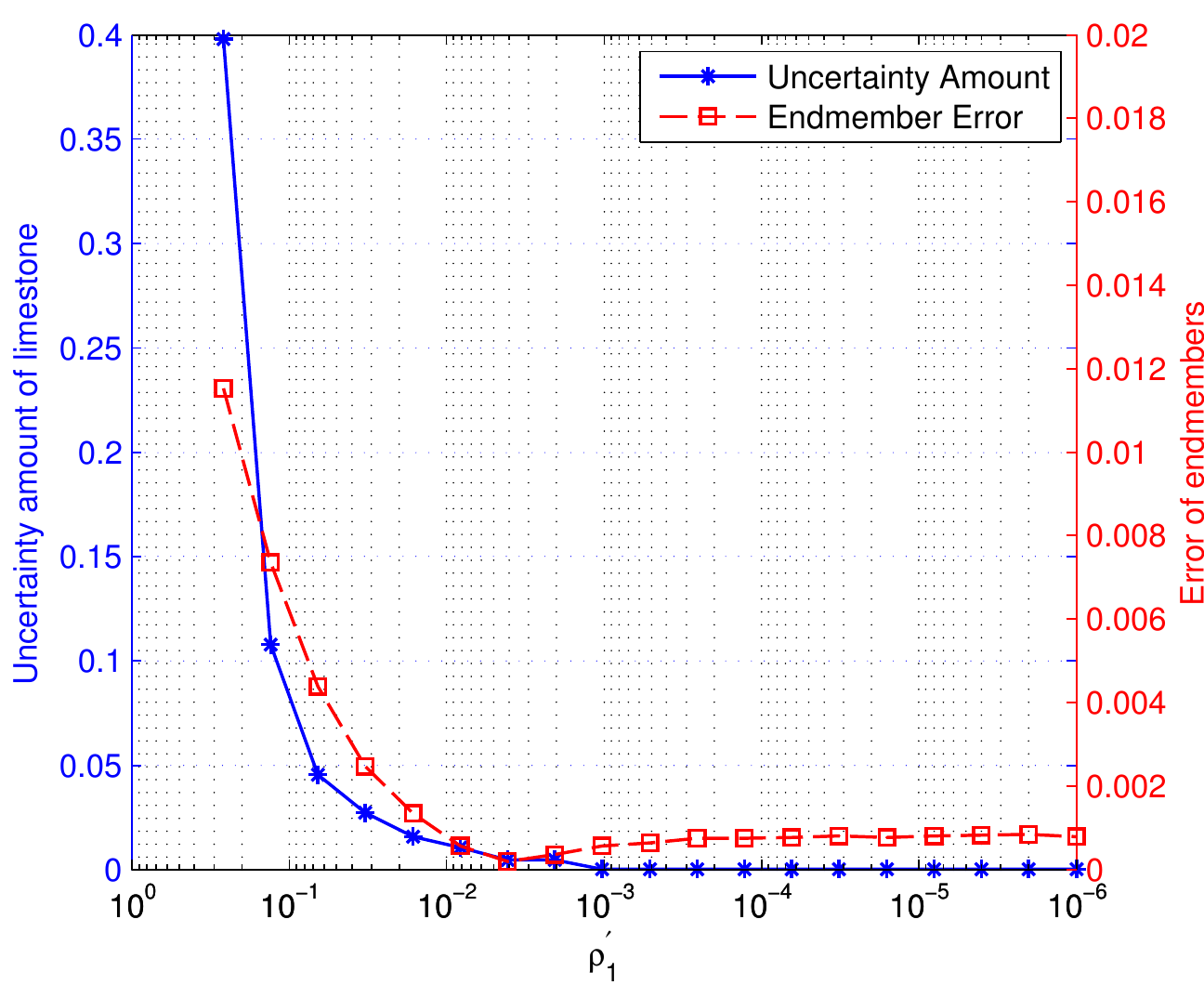}
\par\end{centering}

\caption{Effect of pairwise closeness ($\rho_{1}^{\prime}$) on the uncertainty
amount of limestone (blue solid line, scale on the left) and the error
of endmembers (red dashed line, scale on the right) for a synthetic
image with SNR 40dB. The minimal error corresponds to the starting
point where the uncertainty amount drops to a stable value.}

\label{fig:rho1_uncer}
\end{figure}

Figure~\ref{fig:uncer_range_w/_large_rho1} shows the uncertainty
range with close endmembers when $\rho_{1}^{\prime}=0.1$. From Figure~\ref{fig:rho1_uncer},
the endmembers are actually inside the pixel cloud since it is greater
than the optimal value. We can see that not only the uncertainty amount
reflects the distance to the ground truth, the uncertainty direction
also reflects the distortion of the estimated endmembers. Combing
these pieces of information, the uncertainty range is able to cover
the ground truth for each endmember. Therefore the uncertainty estimated
can serve as a prediction of the endmember error in this case, given
endmembers estimated with a sufficient closeness constraint.

\begin{figure}
\begin{centering}
\includegraphics[width=1\textwidth]{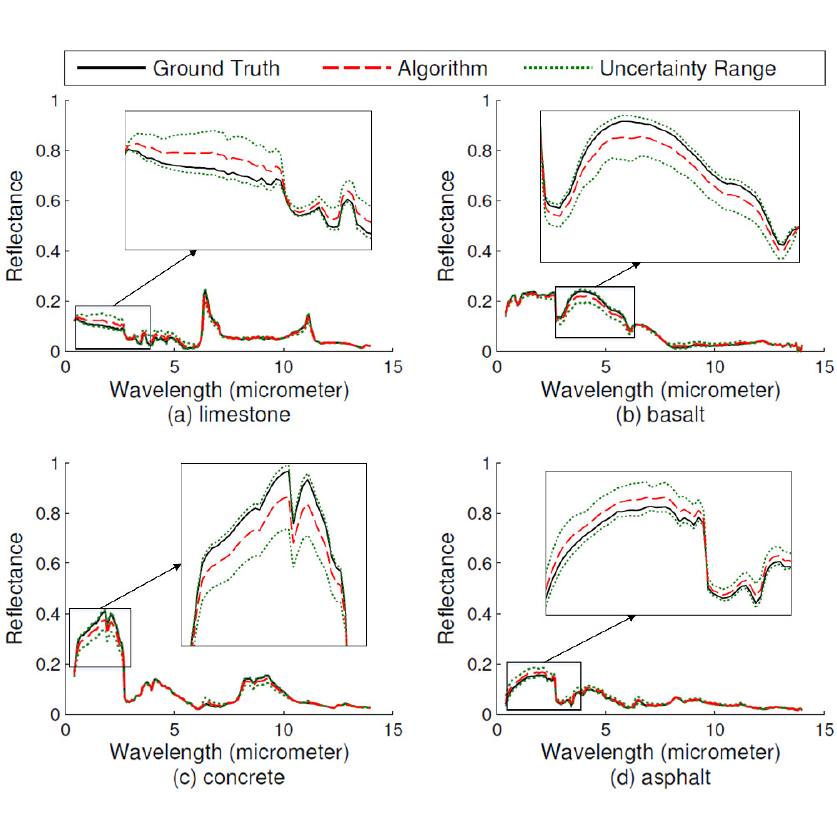}
\par\end{centering}

\caption{Uncertainty ranges with endmembers estimated with $\rho_{1}^{\prime}=0.1$
for the synthetic image tested in Figure~\ref{fig:rho1_uncer}. The
uncertainty ranges cover the ground truth given endmembers estimated
with slight biases.}

\label{fig:uncer_range_w/_large_rho1}
\end{figure}

\subsection{Pavia University}

The SCM algorithm is also applied to the Pavia University dataset,
which was recorded by the Reflective Optics System Imaging Spectrometer
(ROSIS) during a flight over Pavia, northern Italy. It is a 340 by
610 image with 103 bands with wavelengths ranging from 430nm to 860nm.
The real spacing is 1.3 meters. The image covers both natural and
urban areas as shown in Figure~\ref{fig:pavia_rgb}. There are 9
materials identified as ground truth by humans. From the pixels identified
as ground truth, average spectra for each material is calculated as
the ground truth endmember signature. Figure~\ref{fig:pavia_gt_endmembers}
shows the ground truth endmembers. From Figure~\ref{fig:pavia_gt_endmembers},
we find that self-blocking bricks and gravel have very similar spectra
and asphalt and bitumen have very similar spectra. So technically,
in this unsupervised unmixing setting, we can use automated algorithms
to distinguish at most 7 endmembers.

\begin{figure}
\begin{centering}
\includegraphics[width=1\textwidth]{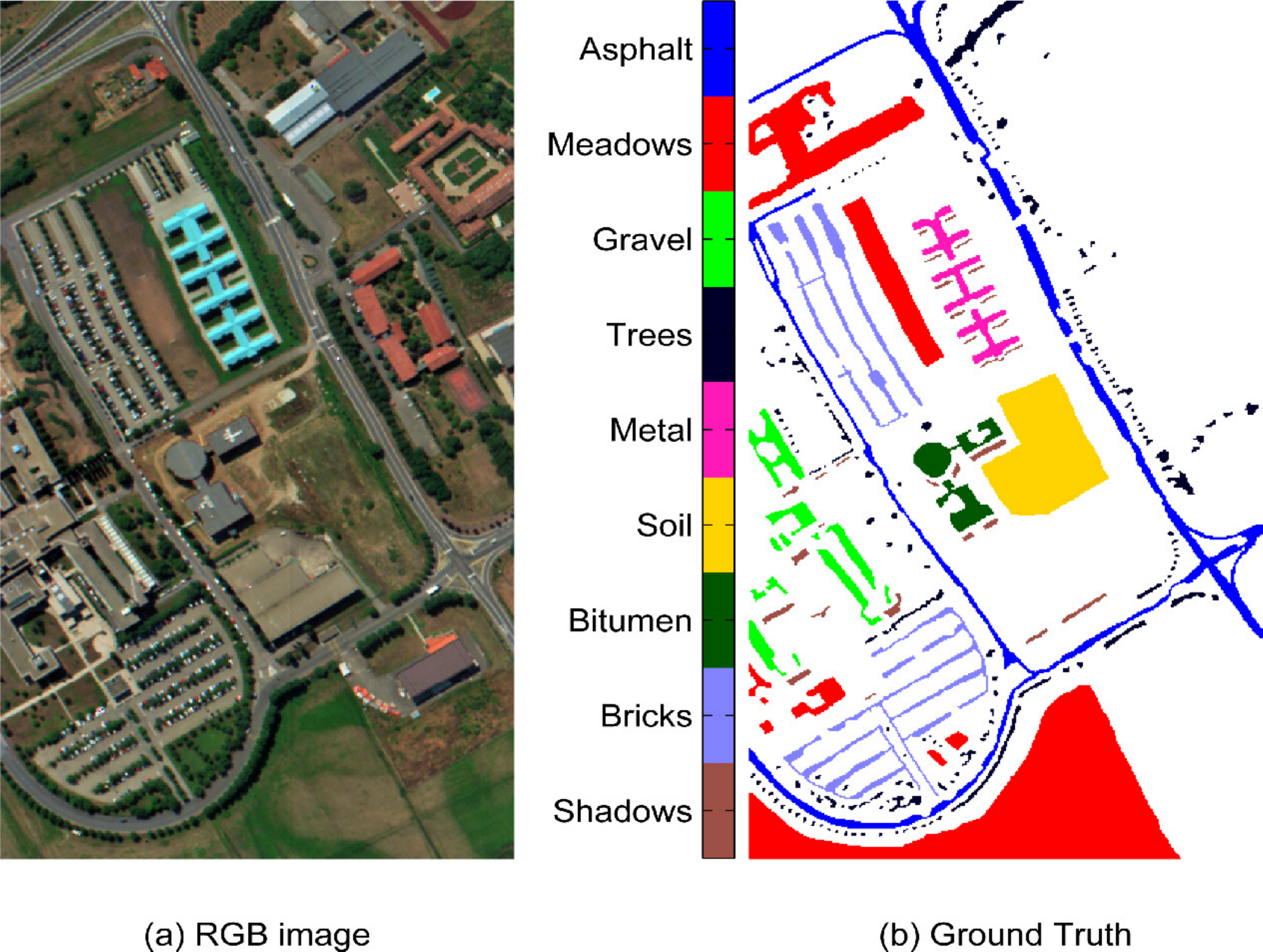}
\par\end{centering}

\caption{RGB image for Pavia University and the ground truth.}

\label{fig:pavia_rgb}
\end{figure}

\begin{figure}
\begin{centering}
\includegraphics[width=1\textwidth]{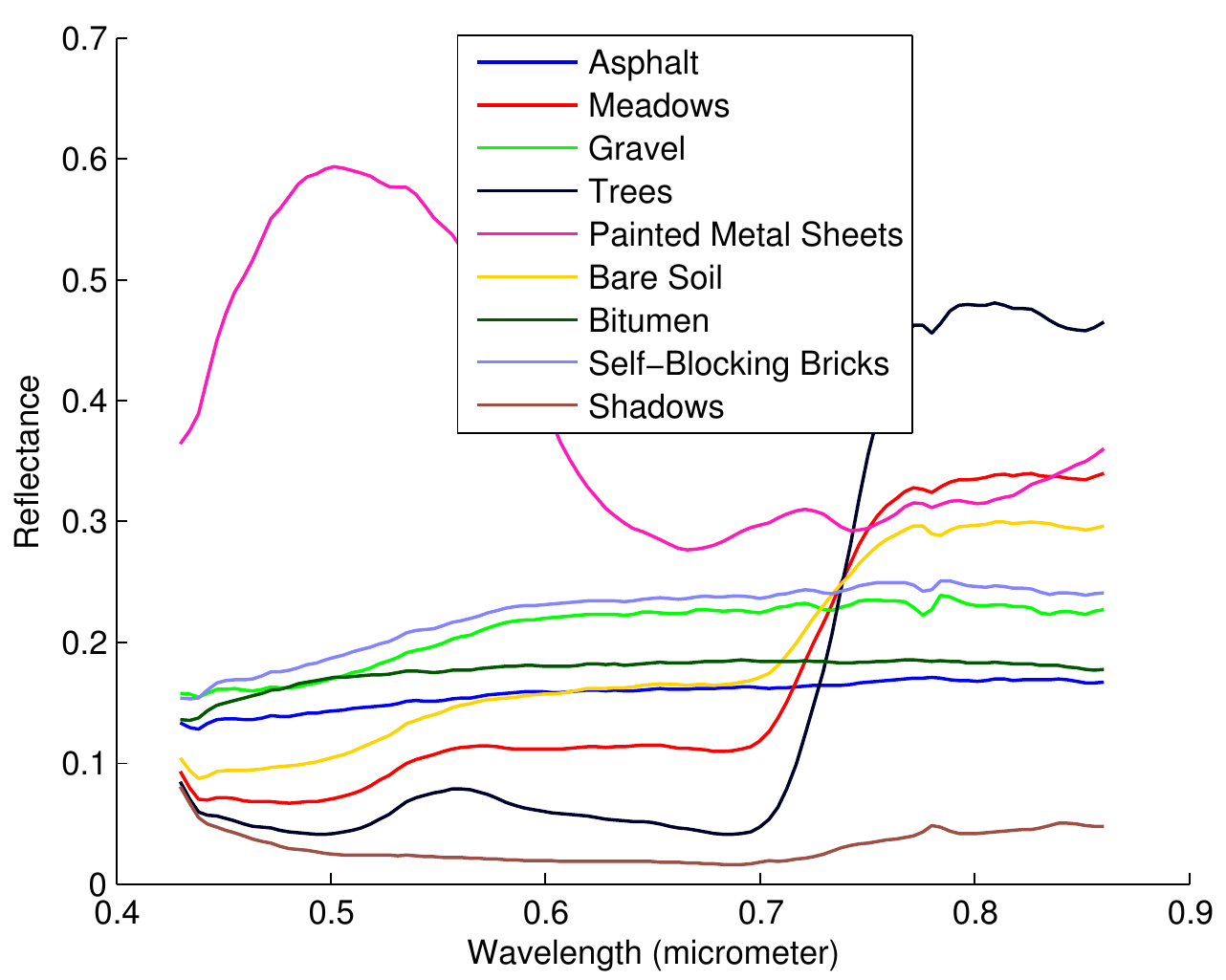}
\par\end{centering}

\caption{Ground truth endmembers for Pavia University. Asphalt and Bitumen
have similar spectral signatures. Gravel and Self-blocking bricks
have similar spectral signatures.}

\label{fig:pavia_gt_endmembers}
\end{figure}

We run SCM, NCM, PCOMMEND on this dataset with 7 endmembers (PCOMMEND
with 6 endmembers as suggested in \cite{zare2013piecewise}). The
parameters for SCM are $\beta_{2}^{\prime}=0.02$, $\rho_{1}^{\prime}=0.05$.
Two materials, gravel and bitumen, are excluded in the comparison
because they are attributed to self-blocking bricks and asphalt respectively.
Figure~\ref{fig:pavia_abundances}(a) shows the abundance maps from
SCM. When compared to the ground truth in Figure~\ref{fig:pavia_rgb},
we can see that the materials are asphalt (bitumen), meadows, trees,
painted metal sheets, bare soil, self-blocking bricks (gravel) and
shadows respectively. The abundance maps of NCM are shown in Figure~\ref{fig:pavia_abundances}(b).
We observe that without the spatial information and the sparsity promoting
effect, the abundance maps present scattered dots within a pure material
region.

\begin{figure}
\begin{centering}
\includegraphics[height=0.95\textheight]{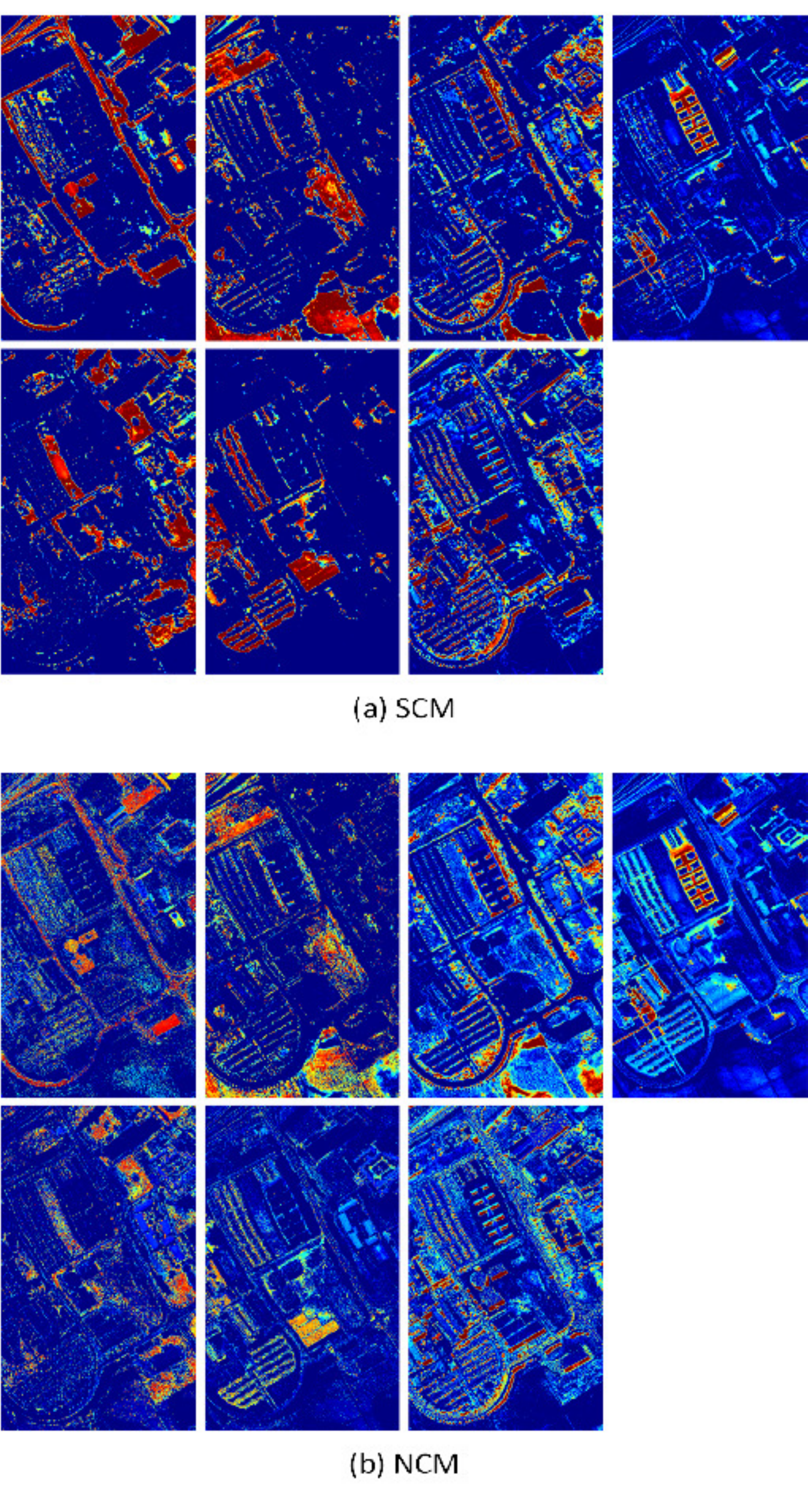}
\par\end{centering}

\caption{Abundance maps from SCM and NCM for Pavia University. The identified
materials are asphalt (bitumen), meadows, trees, painted metal sheets,
bare soil, self-blocking bricks (gravel), shadows respectively.}

\label{fig:pavia_abundances}
\end{figure}

Figure~\ref{fig:pavia_endmember_comparison} shows the resulting
endmember spectra from SCM, NCM, PCOMMEND versus the corresponding
ground truth endmember. We also computed the errors for these endmembers
and the result is shown in Table~\ref{table:pavia_endmembers}. From
these results, we see that for the 7 identified endmembers, the meadows
are missing in PCOMMEND (and this is attributed to wrong ground truth
information regarding meadows). Also, SCM matches the asphalt, meadows,
trees, bare soil and self-blocking bricks best while NCM matches the
painted metal sheets best and PCOMMEND matches the shadows best. The
statistics show that SCM performed best overall (except for a caveat
that the meadows endmembers should be further investigated due to
a discrepancy in the ground-truth).

\begin{figure*}
\begin{centering}
\includegraphics[width=1\textwidth]{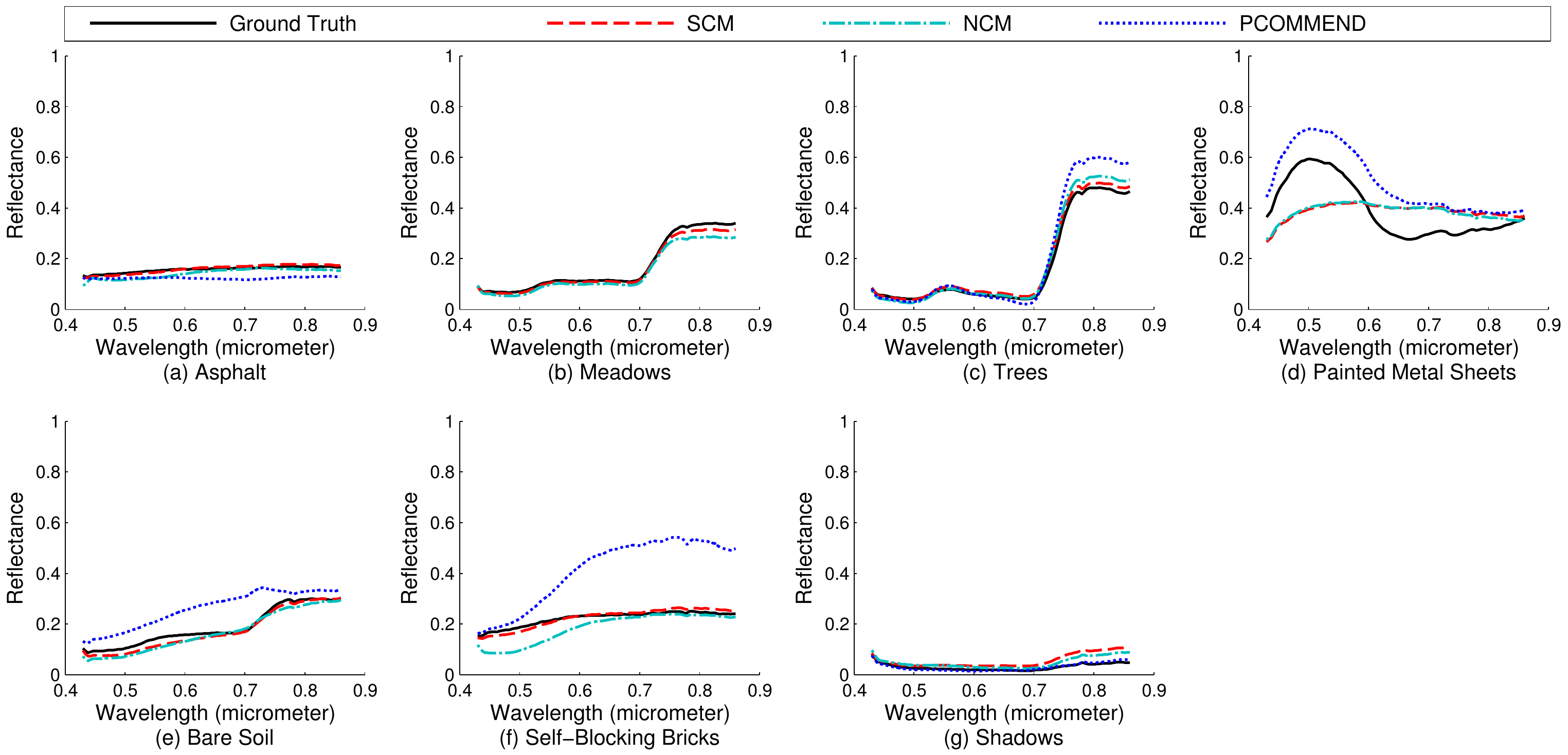}
\par\end{centering}

\caption{Qualitative comparison of endmembers among all the algorithms for
Pavia University.}

\label{fig:pavia_endmember_comparison}
\end{figure*}

\begin{table}

\caption{Quantitative comparison of endmembers among all the algorithms for
Pavia University.}

\begin{centering}
\begin{tabular}{|c|c|c|c|}
\hline 
Error & SCM & NCM & PCOMMEND\tabularnewline
\hline 
\hline 
Asphalt & \textbf{0.0064} & 0.0149 & 0.0335\tabularnewline
\hline 
Meadows & \textbf{0.0095} & 0.0227 & -\tabularnewline
\hline 
Trees & \textbf{0.0135} & 0.0197 & 0.0419\tabularnewline
\hline 
Painted Metal Sheets & 0.1007 & \textbf{0.0947} & 0.1096\tabularnewline
\hline 
Bare Soil & \textbf{0.0144} & 0.0213 & 0.0748\tabularnewline
\hline 
Self-Blocking Bricks & \textbf{0.0105} & 0.0378 & 0.1885\tabularnewline
\hline 
Shadows & 0.0251 & 0.0170 & \textbf{0.0043}\tabularnewline
\hline 
Average & \textbf{0.0257} & 0.0326 & 0.0754\tabularnewline
\hline 
\end{tabular}
\par\end{centering}

\label{table:pavia_endmembers}
\end{table}

The uncertainty ranges of endmembers from SCM for Pavia University
along with the ground truth are shown in Figure~\ref{fig:pavia_scm_uncer}.
We see that for those well estimated endmembers, the uncertainties
are so small that the endmembers coincide with the uncertainty ranges.
For the largely biased endmember of painted metal sheets, the uncertainty
is also large such that it nearly covers the ground truth. For the
shadows, the SCM estimated endmember deviates from the ground truth
at the right end and the uncertainty range also features a large gap
at the right end.

\begin{figure*}
\begin{centering}
\includegraphics[width=1\textwidth]{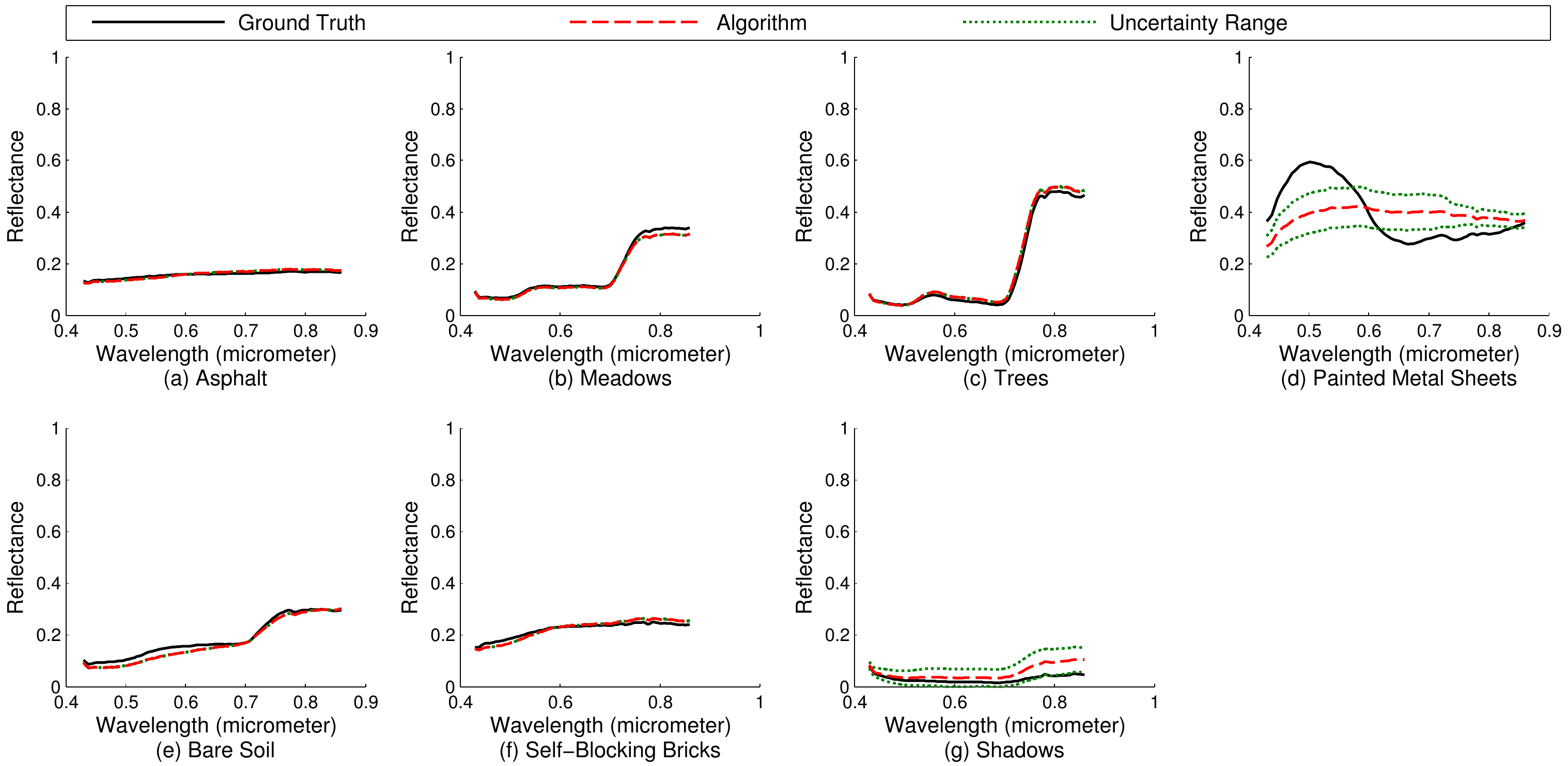}
\par\end{centering}

\caption{Uncertainty ranges of endmembers estimated from SCM for Pavia University.}

\label{fig:pavia_scm_uncer}
\end{figure*}

\subsection{Indian Pines}

The Indian Pines dataset was collected by the Airborne Visible/Infrared
Imaging Spectrometer (AVIRIS) sensor over the Indian Pines test site
in Northwestern Indiana. It is a 145 by 145 image with 220 bands in
the wavelength range 0.4 - 2.5$\mu$m. Most areas of the image are
agriculture, forest and other vegetation, except two highways, a railway
line, and a few buildings. Figure~\ref{fig:indian_pines_rgb} shows
the RGB image of this dataset and the ground truth materials. Figure~\ref{fig:indian_pines_gt}
shows the ground truth endmembers. We can see that the ground truth
distinguishes the pixels into 16 classes, where most have quite similar
spectra. 

\begin{figure*}
\begin{centering}
\includegraphics[width=1\textwidth]{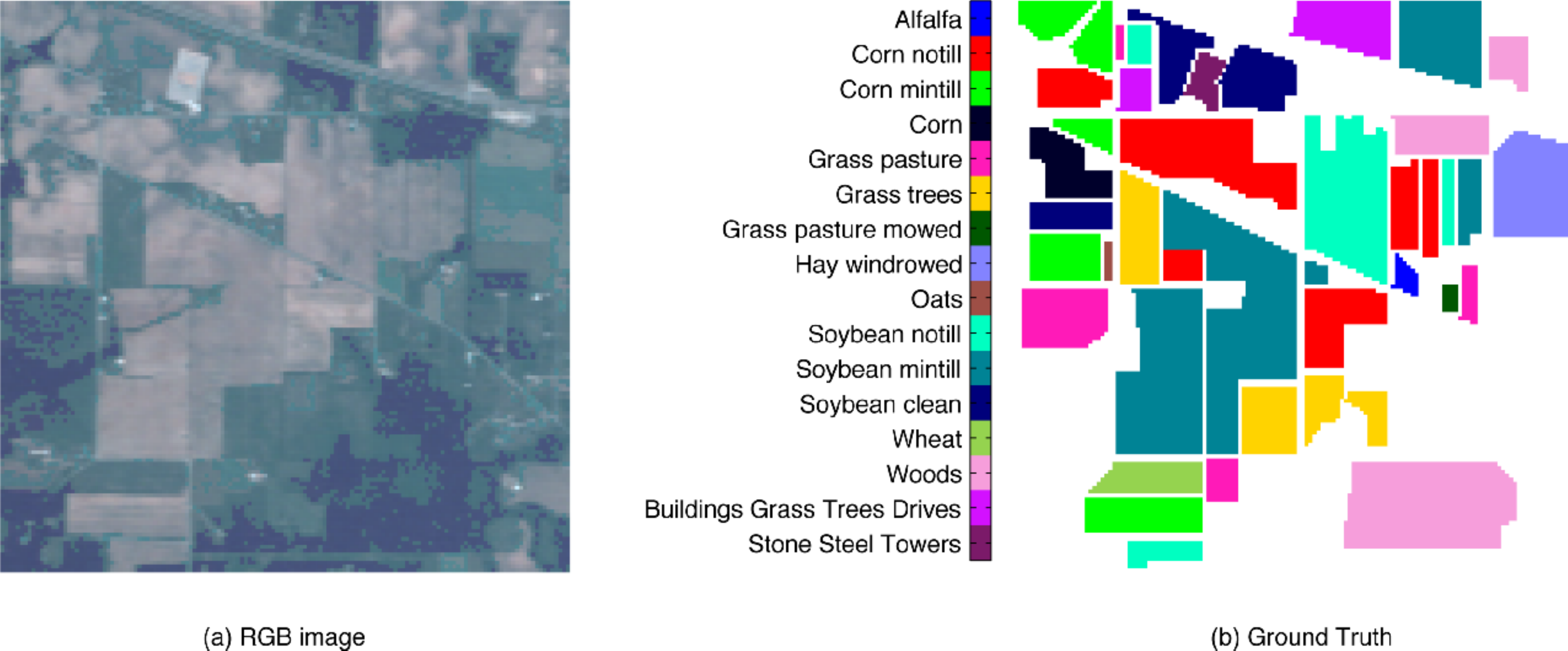}
\par\end{centering}

\caption{RGB image for Indian Pines and the ground truth.}

\label{fig:indian_pines_rgb}
\end{figure*}

\begin{figure}
\begin{centering}
\includegraphics[width=1\textwidth]{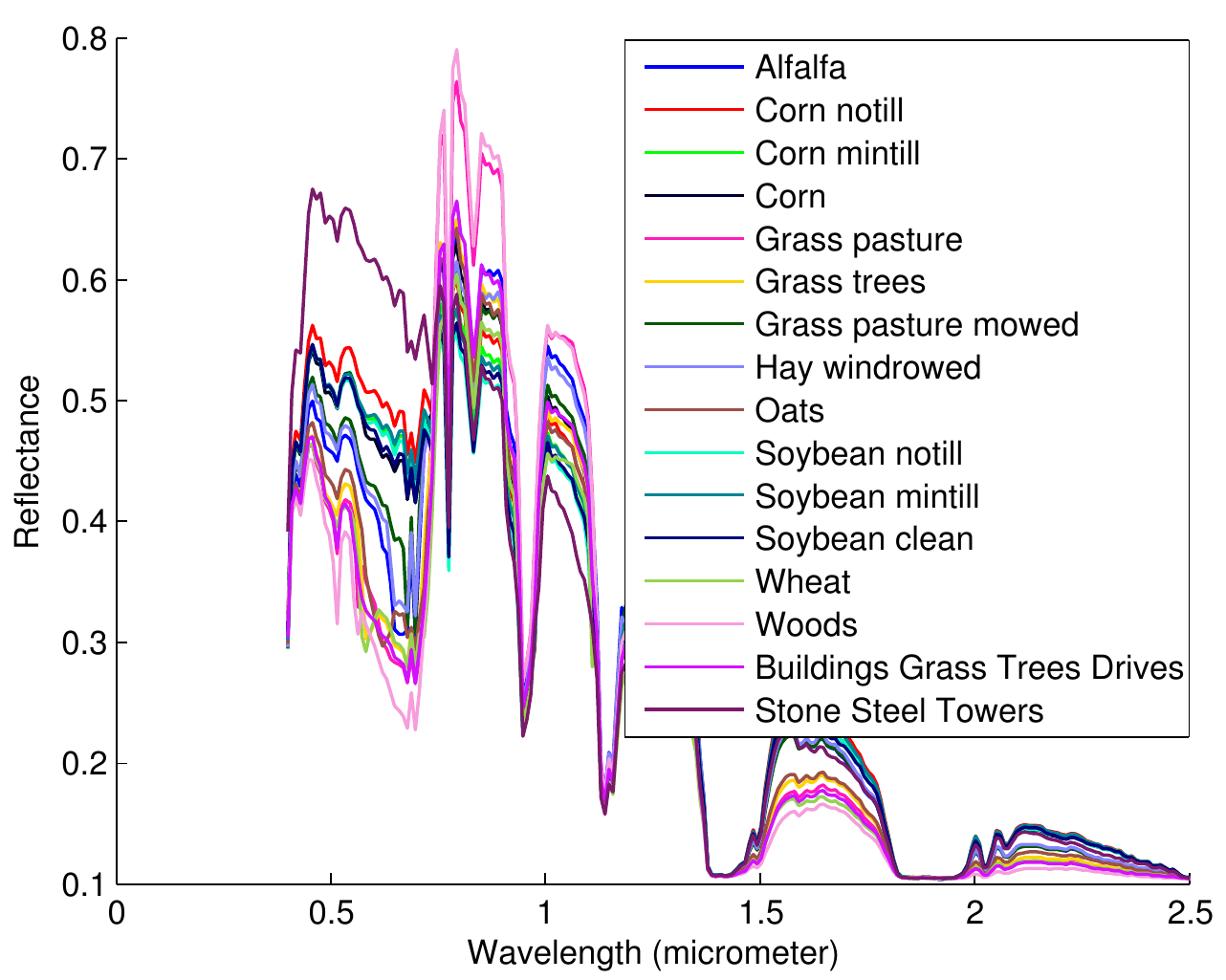}
\par\end{centering}

\caption{Ground truth endmembers for Indian Pines.}

\label{fig:indian_pines_gt}
\end{figure}

From the regions with visibly different colors in Figure~\ref{fig:indian_pines_rgb}
and the noticeably different endmember spectra in Figure~\ref{fig:indian_pines_gt},
we set $M=4$ for all the algorithms. The parameters for SCM were
set to $\beta_{2}^{\prime}=0.005$, $\rho_{1}^{\prime}=0.05$. The
parameters for PCOMMEND were set to 2 clusters with 2 endmembers in
each cluster. The results are shown in Figure~\ref{fig:indian_pines_abundances}.
The abundance maps show that NCM, PCOMMEND present inconsistent abundances
in a ground truth region, e.g. soybean mintill, while SCM presents
consistent abundances in the same region.

\begin{figure*}
\begin{centering}
\includegraphics[width=1\textwidth]{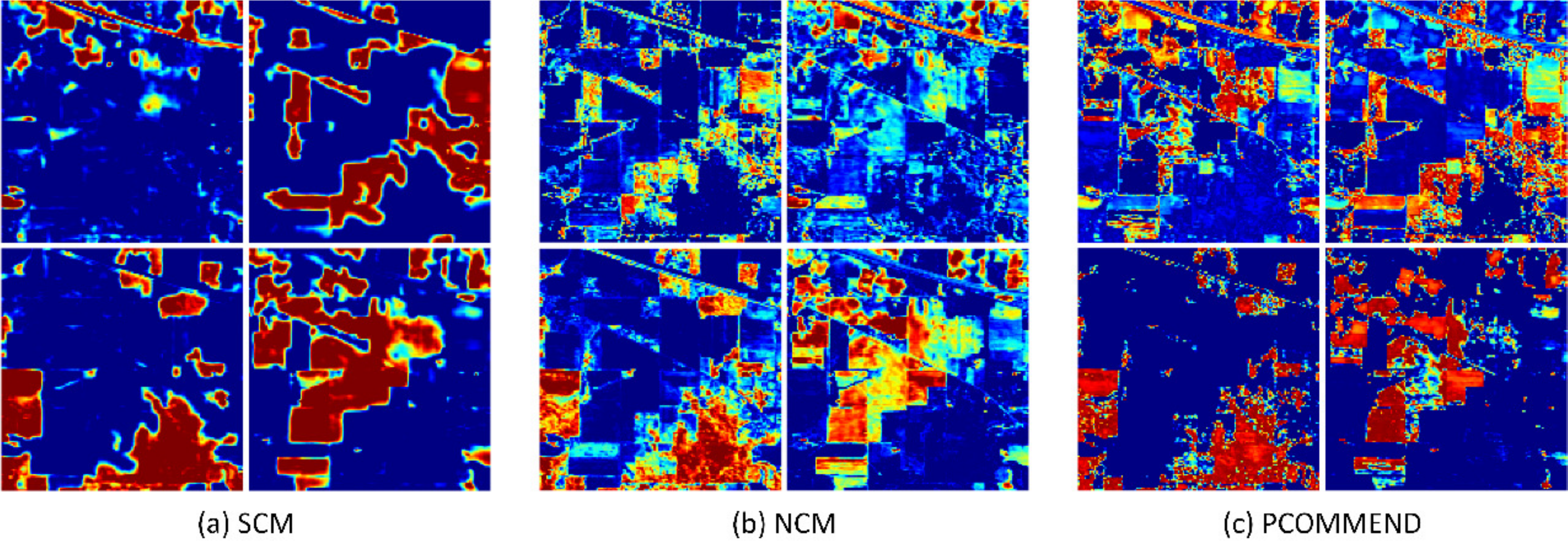}
\par\end{centering}

\caption{Abundance maps from SCM, NCM and PCOMMEND for Indian Pines.}

\label{fig:indian_pines_abundances}
\end{figure*}

\section{Discussion and Conclusion}

In this paper we presented a spatial compositional model (SCM) for
linearly unmixing hyperspectral images. The benefits of our model
include calculating the full likelihood for uncertainty estimation,
a weighted smoothness term on neighboring abundances, and a simple
and efficient algorithm that also estimates the endmember uncertainty.
The algorithm usually converges within 100 iterations and takes about
2 minutes to process the Pavia University dataset on a laptop with
an i7 CPU. The results on synthetic and real datasets show that the
estimated endmembers are more accurate than NCM and PCOMMEND. Moreover,
the uncertainty encoded by the covariance matrix shows that its range
can predict error when the estimated endmembers are inside the pixels.
Future work will focus on physically realistic models for hyperspectral
unmixing.

\section*{Appendix}

We show that the objective function (\ref{eq:E6}) can be approximated
by (\ref{eq:E7}) for minimization with respect to $\mathbf{A},\,\mathbf{R}$
in this Appendix, i.e.
\[
0<\gamma\mathbf{z}^{T}\mathbf{Q}^{-1}\mathbf{z}\ll\gamma\Vert\mathbf{Y}-\mathbf{A}\mathbf{R}\Vert_{F}^{2},
\]
\[
0<\text{log}\left|\mathbf{Q}\right|-\sum_{j=1}^{M}\text{log}\left|\mathbf{S}_{j}\right|\ll\gamma\Vert\mathbf{Y}-\mathbf{A}\mathbf{R}\Vert_{F}^{2},
\]
given $\gamma,\,\left\{ \mathbf{S}_{j}\right\} $ fixed at some optimal
values. To be specific, both $\gamma\mathbf{z}^{T}\mathbf{Q}^{-1}\mathbf{z}$
and $\text{log}\left|\mathbf{Q}\right|-\sum_{j=1}^{M}\text{log}\left|\mathbf{S}_{j}\right|$
are in $O\left(MB\right)$ while $\gamma\Vert\mathbf{Y}-\mathbf{A}\mathbf{R}\Vert_{F}^{2}$
is in $O\left(NB\right)$. In the applications---Pavia University
($N=207400$, $M=7$) and Indian Pines ($N=21025$, $M=4$)---$\gamma\mathbf{z}^{T}\mathbf{Q}^{-1}\mathbf{z}$
is about 9\% and 2\% of the least squares term respectively while
$\text{log}\left|\mathbf{Q}\right|-\sum_{j=1}^{M}\text{log}\left|\mathbf{S}_{j}\right|$
is about 0.02\% and 0.001\% respectively.

We first show that $\mathbf{z}^{T}\mathbf{Q}^{-1}\mathbf{z}$ (positive
because $\mathbf{Q}\in\text{SPD}\left(MB\right)$) is negligible compared
to $\Vert\mathbf{Y}-\mathbf{A}\mathbf{R}\Vert_{F}^{2}$. Assume $\mathbf{A}^{T}\mathbf{A}$
is nonsingular (hence $\mathbf{A}^{T}\mathbf{A}\otimes\mathbf{I}_{B}\in\text{SPD}\left(MB\right)$),
from the inequality in Lemma~\ref{lemma1} (given at the end of this
Appendix), we have
\begin{eqnarray*}
\mathbf{z}^{T}\mathbf{Q}^{-1}\mathbf{z} & = & \mathbf{z}^{T}\left(\left[\delta_{ij}\mathbf{S}_{j}\right]+\mathbf{A}^{T}\mathbf{A}\otimes\mathbf{I}_{B}\right)^{-1}\mathbf{z}\\
 & < & \mathbf{z}^{T}\left(\mathbf{A}^{T}\mathbf{A}\otimes\mathbf{I}_{B}\right)^{-1}\mathbf{z}\\
 & = & \mathbf{z}^{T}\left(\left(\mathbf{V}\boldsymbol{\Lambda}\right)^{-T}\otimes\mathbf{I}_{B}\right)\left(\left(\mathbf{V}\boldsymbol{\Lambda}\right)^{-1}\otimes\mathbf{I}_{B}\right)\mathbf{z}\\
 & = & \Vert\text{vec}\left(\left(\mathbf{Y}-\mathbf{A}\mathbf{R}\right)^{T}\mathbf{A}\left(\mathbf{V}\boldsymbol{\Lambda}\right)^{-T}\right)\Vert^{2}\\
 & = & \Vert\text{vec}\left(\left(\mathbf{Y}-\mathbf{A}\mathbf{R}\right)^{T}\mathbf{U}\right)\Vert^{2}\\
 & = & \Vert\mathbf{U}^{T}\left(\mathbf{Y}-\mathbf{A}\mathbf{R}\right)\Vert_{F}^{2}
\end{eqnarray*}
where $\mathbf{A}=\mathbf{U}\boldsymbol{\Lambda}\mathbf{V}^{T}$,
$\mathbf{U}\in\mathbb{R}^{N\times M}$, $\boldsymbol{\Lambda}\in\mathbb{R}^{M\times M}$,
$\mathbf{V}\in\mathbb{R}^{M\times M}$ is the compact \emph{singular
value decomposition} (SVD) of $\mathbf{A}$. Since $\mathbf{U}^{T}$
is part of an orthogonal matrix, $\mathbf{U}^{T}\left(\mathbf{Y}-\mathbf{A}\mathbf{R}\right)$
can be seen as rotating the columns of $\mathbf{Y}-\mathbf{A}\mathbf{R}$
and picking only $M$ elements of the rotated vectors, it is trivial
compared to $\Vert\mathbf{Y}-\mathbf{A}\mathbf{R}\Vert_{F}^{2}$ which
has $N$ elements for each column of $\mathbf{Y}-\mathbf{A}\mathbf{R}$.

Second we can show that $\text{log}\left|\mathbf{Q}\right|-\sum_{j=1}^{M}\text{log}\left|\mathbf{S}_{j}\right|>0$
and it is also negligible compared to $\gamma\Vert\mathbf{Y}-\mathbf{A}\mathbf{R}\Vert_{F}^{2}$.
The positivity arises from \emph{Weyl's inequality} (Theorem 4.3.1
in \cite{horn1985matrix}) as the eigenvalues of $\mathbf{Q}$ are
greater than those of $\left[\delta_{ij}\mathbf{S}_{j}\right]$. Note
that 
\begin{eqnarray*}
\text{log}\left|\mathbf{Q}\right| & = & \text{log}\left|\left[\delta_{ij}\mu^{2}\boldsymbol{\Sigma}_{j}^{-1}\right]+\mathbf{A}^{T}\mathbf{A}\otimes\mathbf{I}_{B}\right|\\
 & = & \text{log}\mu^{2MB}\left|\left[\delta_{ij}\boldsymbol{\Sigma}_{j}^{-1}\right]+\mu^{-2}\mathbf{A}^{T}\mathbf{A}\otimes\mathbf{I}_{B}\right|\\
 & = & -MB\text{log}\gamma+\text{log}\left|\left[\delta_{ij}\boldsymbol{\Sigma}_{j}^{-1}\right]+\gamma\mathbf{A}^{T}\mathbf{A}\otimes\mathbf{I}_{B}\right|,
\end{eqnarray*}
and
\begin{eqnarray*}
-\sum_{j=1}^{M}\text{log}\left|\mathbf{S}_{j}\right| & = & -\sum_{j=1}^{M}\text{log}\left|\mu^{2}\boldsymbol{\Sigma}_{j}^{-1}\right|\\
 & = & -\sum_{j=1}^{M}\left\{ \text{log}\mu^{2B}+\text{log}\left|\boldsymbol{\Sigma}_{j}^{-1}\right|\right\} \\
 & = & MB\text{log}\gamma-\sum_{j=1}^{M}\text{log}\left|\boldsymbol{\Sigma}_{j}^{-1}\right|.
\end{eqnarray*}
Let $\sigma_{j1},...,\sigma_{jB}$ be the the eigenvalues of $\boldsymbol{\Sigma}_{j}^{-1}$
in ascending order and $\lambda_{1},...,\lambda_{M}$ be the eigenvalues
of $\mathbf{A}^{T}\mathbf{A}$ in ascending order (so the eigenvalues
of $\mathbf{A}^{T}\mathbf{A}\otimes\mathbf{I}_{B}$ are $\lambda_{1},...,\lambda_{M}$
duplicated by $B$ times). The sum of the above two expansions leads
to 
\begin{alignat*}{1}
 & \text{log}\left|\mathbf{Q}\right|-\sum_{j=1}^{M}\text{log}\left|\mathbf{S}_{j}\right|\\
= & \text{log}\left|\left[\delta_{ij}\boldsymbol{\Sigma}_{j}^{-1}\right]+\gamma\mathbf{A}^{T}\mathbf{A}\otimes\mathbf{I}_{B}\right|-\sum_{j=1}^{M}\text{log}\left|\boldsymbol{\Sigma}_{j}^{-1}\right|\\
\leq & \text{log}\prod_{j=1}^{M}\prod_{k=1}^{B}\left(\sigma_{jk}+\gamma\lambda_{M}\right)-\sum_{j=1}^{M}\text{log}\prod_{k=1}^{B}\sigma_{jk}\\
= & \sum_{j=1}^{M}\sum_{k=1}^{B}\text{log}\left(1+\gamma\lambda_{M}/\sigma_{jk}\right)
\end{alignat*}
where Weyl's inequality for the eigenvalues is again used. Given that
the reflectances of the endmember signatures are in the range $[0,1]$
(a real hyperspectral image is usually normalized during preprocessing),
the endmember covariance matrix should have $\sigma_{jk}$ bounded
from below. Although the inside of the logarithm is large, the logarithm
makes it limited. Compared with $\gamma\Vert\mathbf{Y}-\mathbf{A}\mathbf{R}\Vert_{F}^{2}\approx NB$
with $\gamma$ given in (\ref{eq:optimalGamma1}), $\sum_{j=1}^{M}\sum_{k=1}^{B}\text{log}\left(1+\gamma\lambda_{M}/\sigma_{jk}\right)$
is negligible considering $M\ll N$.
\begin{lem}
\label{lemma1}Let $\mathbf{A}\in\text{SPD}\left(n\right)$, $\mathbf{B}\in\text{SPD}\left(n\right)$,
for any nonzero $\mathbf{x}\in\mathbb{R}^{n}$, then $\mathbf{x}^{T}\left(\mathbf{A}+\mathbf{B}\right)^{-1}\mathbf{x}<\mathbf{x}^{T}\mathbf{A}^{-1}\mathbf{x}$.\end{lem}
\begin{proof}
Let $\mathbf{A}=\mathbf{U}\boldsymbol{\Sigma}\mathbf{U}^{T}$, $\mathbf{B}=\mathbf{V}\boldsymbol{\Lambda}\mathbf{V}^{T}$
be the eigendecomposition of $\mathbf{A}$ and $\mathbf{B}$ respectively.
Then 
\[
\mathbf{x}^{T}\mathbf{A}^{-1}\mathbf{x}=\mathbf{x}^{T}\left(\mathbf{U}\boldsymbol{\Sigma}\mathbf{U}^{T}\right)^{-1}\mathbf{x}=\mathbf{y}^{T}\boldsymbol{\Sigma}^{-1}\mathbf{y}
\]
where $\mathbf{y}=\mathbf{U}^{T}\mathbf{x}$, while 
\begin{eqnarray*}
\mathbf{x}^{T}\left(\mathbf{A}+\mathbf{B}\right)^{-1}\mathbf{x} & = & \mathbf{x}^{T}\left(\mathbf{U}\left(\boldsymbol{\Sigma}+\mathbf{Q}\boldsymbol{\Lambda}\mathbf{Q}^{T}\right)\mathbf{U}^{T}\right)^{-1}\mathbf{x}\\
 & = & \mathbf{y}^{T}\left(\boldsymbol{\Sigma}+\mathbf{Q}\boldsymbol{\Lambda}\mathbf{Q}^{T}\right)^{-1}\mathbf{y}
\end{eqnarray*}
where $\mathbf{Q}=\mathbf{U}^{T}\mathbf{V}$. By the Woodbury identity,
\[
\left(\boldsymbol{\Sigma}+\mathbf{Q}\boldsymbol{\Lambda}\mathbf{Q}^{T}\right)^{-1}=\boldsymbol{\Sigma}^{-1}-\boldsymbol{\Sigma}^{-1}\mathbf{Q}\mathbf{C}\mathbf{Q}^{T}\boldsymbol{\Sigma}^{-1}
\]
where
\[
\mathbf{C}=\left(\boldsymbol{\Lambda}^{-1}+\mathbf{Q}^{T}\boldsymbol{\Sigma}^{-1}\mathbf{Q}\right)^{-1},
\]
we have 
\[
\mathbf{y}^{T}\left(\boldsymbol{\Sigma}+\mathbf{Q}\boldsymbol{\Lambda}\mathbf{Q}^{T}\right)^{-1}\mathbf{y}=\mathbf{y}^{T}\boldsymbol{\Sigma}^{-1}\mathbf{y}-\mathbf{z}^{T}\mathbf{C}\mathbf{z}
\]
where $\mathbf{z}=\mathbf{Q}^{T}\boldsymbol{\Sigma}^{-1}\mathbf{y}$.
Because $\mathbf{C}\in\text{SPD}\left(n\right)$ (since $\boldsymbol{\Lambda}^{-1}\in\text{SPD}\left(n\right)$
and $\mathbf{Q}^{T}\boldsymbol{\Sigma}^{-1}\mathbf{Q}\in\text{SPD}\left(n\right)$)
and $\mathbf{z}$ is nonzero, $\mathbf{z}^{T}\mathbf{C}\mathbf{z}>0$.
Then we have $\mathbf{x}^{T}\left(\mathbf{A}+\mathbf{B}\right)^{-1}\mathbf{x}<\mathbf{x}^{T}\mathbf{A}^{-1}\mathbf{x}$.
\end{proof}
\bibliographystyle{abbrv}
\bibliography{reference}

\end{document}